\documentclass[conference]{IEEEconf}

\input epsf
\usepackage{graphicx}
\usepackage{multirow}
\usepackage{titlesec}
\usepackage{balance} 
\usepackage{stfloats}
\usepackage{xcolor}
\usepackage{orcidlink}
\renewcommand\thesection{\arabic{section}} 

\renewcommand\thesubsectiondis{\thesection.\arabic{subsection}}

\renewcommand\thesubsubsectiondis{\thesubsectiondis.\arabic{subsubsection}}

\usepackage{graphicx}
\usepackage{multirow}
\usepackage{titlesec}
\usepackage{balance} 
\usepackage{stfloats}
\usepackage{xcolor}
\usepackage{tabularx}
\usepackage[sort&compress,numbers]{natbib}
\usepackage{subcaption}
\usepackage{url}
\usepackage{indentfirst}
\usepackage{hyperref}

\usepackage{nicefrac}
    
\usepackage{graphicx}
\usepackage{diagbox}
\usepackage{amssymb}
\usepackage{amsthm}
\usepackage{amsmath}
\usepackage{xcolor}
\usepackage{multirow}
\usepackage{comment}
\usepackage{soul,color}
\usepackage{cancel}
\usepackage{tablefootnote}
\usepackage[colorinlistoftodos,prependcaption,textsize=small]{todonotes}
\usepackage[ruled, linesnumbered]{algorithm2e}
\usepackage{soul, color, xcolor}
\usepackage{amsmath}
\newtheorem{thm}{Theorem}

\newtheorem{cor}{Corollary}
\newtheorem{Definition}{Definition}
\newtheorem{Property}{Property}
\IEEEoverridecommandlockouts

\def\BibTeX{{\rm B\kern-.05em{\sc i\kern-.025em b}\kern-.08em
    T\kern-.1667em\lower.7ex\hbox{E}\kern-.125emX}}
    
\begin{document}

\title{
Scalable and 
Precise Patch Robustness Certification for  
Deep Learning Models
with Top-$k$ Predictions
}

 \author{
Qilin Zhou$^{1}$\,\orcidlink{0000-0003-2289-9849}, Haipeng Wang$^{1}$\,\orcidlink{0000-0002-7410-393X}, Zhengyuan Wei$^{2}$\,\orcidlink{0000-0001-5966-1338}, and W.K. Chan$^{1,*}$\,\orcidlink{0000-0001-7726-6235}\\
    \{qilin.zhou, haipeng.wang\}@my.cityu.edu.hk,
    zywei@eduhk.hk,
    wkchan@cityu.edu.hk\\
\IEEEauthorblockA{$^1$City University of Hong Kong, Hong Kong  $^2$The Education University of Hong Kong, Hong Kong}
 	\normalsize *corresponding author

     \thanks{
    This research is supported in part by CityU MF\_EXT (project no. 9678180) and 
    a central fund from EdUHK (project no. GIET-02A14).
    }
 }



\maketitle
\begin{abstract}
Patch robustness certification is an emerging verification approach for defending against adversarial patch attacks with \emph{provable} guarantees for deep learning systems.
Certified recovery techniques guarantee the prediction of the \emph{sole} true label of a certified sample.
However, existing techniques, if applicable to top-$k$ predictions, commonly conduct pairwise comparisons on those votes between labels, failing to certify the sole true label within the top $k$ prediction labels \emph{precisely} due to the inflation on the number of votes controlled by the attacker (i.e., attack budget);
yet enumerating all combinations of vote allocation suffers from the combinatorial explosion problem.
We propose CostCert, a novel, scalable, and precise
voting-based certified recovery defender.
CostCert verifies the true label of a sample within the top $k$ predictions without pairwise comparisons and combinatorial explosion through a novel design: whether the attack budget on the sample is infeasible to cover the smallest total additional votes on top of the votes uncontrollable by the attacker to exclude the true labels from the top $k$ prediction labels.
Experiments show that CostCert significantly outperforms the current
state-of-the-art defender PatchGuard, such as retaining up to 57.3\% in certified accuracy when the patch size is 96, whereas PatchGuard has already dropped to zero.
\end{abstract}
\begin{IEEEkeywords}
 \itshape Top-k certification,  
 patch attacks,
 robustness,
 worst-case analysis, 
 verification,
 deep learning model
\end{IEEEkeywords}



\section{Introduction}

\label{sec:introduction}

Deep learning (DL) systems 
are susceptible to adversarial attacks \cite{brown2017adversarial,szegedy2013intriguing, liu2020bias}. 
For instances, 
tag-like perturbations 
patched on items can fool online shopping platforms \cite{liu2020bias} (see Fig. \ref{fig:patch_example}), which 
 is an example generally known as a
\textbf{patch attack} 
\cite{brown2017adversarial}: a model for \emph{physical} adversarial attacks
\cite{levine2020randomized}, where attackers can change pixels arbitrarily within a specific region (called a patch region \cite{zhou2024crosscert, xiang2021patchguard, patchcensor}) in a sample. 
Quality assurance on them and improving their robustness are essential. 

\emph{Certified defense} \cite{zhang2020clipped,xiang2021patchguard,xiang2022patchcleanser,levine2020randomized,salman2022certified,chiang2020certified,metzen2021efficient, chen2022towards,li2022vip,patchcensor, zhou2023majority,zhou2024crosscert}, a kind of novel defense based on verification methods on benign samples against patch attacks with \emph{provable} guarantee, is emerging,
unlike empirical defenses \cite{hayes2018visible,naseer2019local} which becomes fragile if attackers know their defense strategies \cite{chiang2020certified}. 
%
In particular,
\textbf{\emph{certified recovery}} \cite{zhang2020clipped,xiang2021patchguard,xiang2022patchcleanser,levine2020randomized,salman2022certified,chiang2020certified,metzen2021efficient, chen2022towards,zhou2023majority,li2022vip} gives certifications to samples that they promise to provide a \textbf{\emph{deterministic guarantee}} that \emph{always return} the \emph{true label} 
even if these samples are arbitrarily perturbed within the patch region sizes (commonly known as the patch sizes) they have analyzed in their verification processes. 

Almost all existing certified recovery defenders are devoted to offering guarantees to the top-1 label only 
\cite{chiang2020certified,levine2020randomized,salman2022certified,chen2022towards,lin2021certified,xiang2021patchguard,metzen2021efficient,zhang2020clipped,xiang2022patchcleanser,xiang2024patchcure,zhou2023majority,li2022vip}.
DL systems with top-$k$ predictions \cite{berrada2018smooth,hsu2021fingat,petersen2022differentiable} return $k$ (where $k \geq 1$) labels (referred to as the top-$k$ labels) for a given sample, for example, the photo-to-product feature in e-commerce apps \cite{liu2020bias} and the photo-to-herb feature in Chinese medicine herbs apps.
Patched images may make the true label fall outside the visible ranges (see Fig.~\ref{fig:patch_example}), potentially pushing some unexpected labels (e.g., counterfeit items or toxic herbs) to list first.
%
Recently, increasing attacks \cite{zhang2022investigating,tursynbek2022geometry} target such an application scenario of top-$k$ predictions, leading to an emerging demand for its robustness certification \cite{jia2020certified, jia2022almost}.

Patch robustness certification for top-$k$ predictions aims to ensure that a sample's true label is among the top-$k$ labels, for $k \geq 1$, predicted by a deep learning model, even when the sample is under patch attacks.




\begin{figure}[bt] 
\centering
\begin{subfigure}[t]{\linewidth}
\centering
\includegraphics[width=\linewidth]{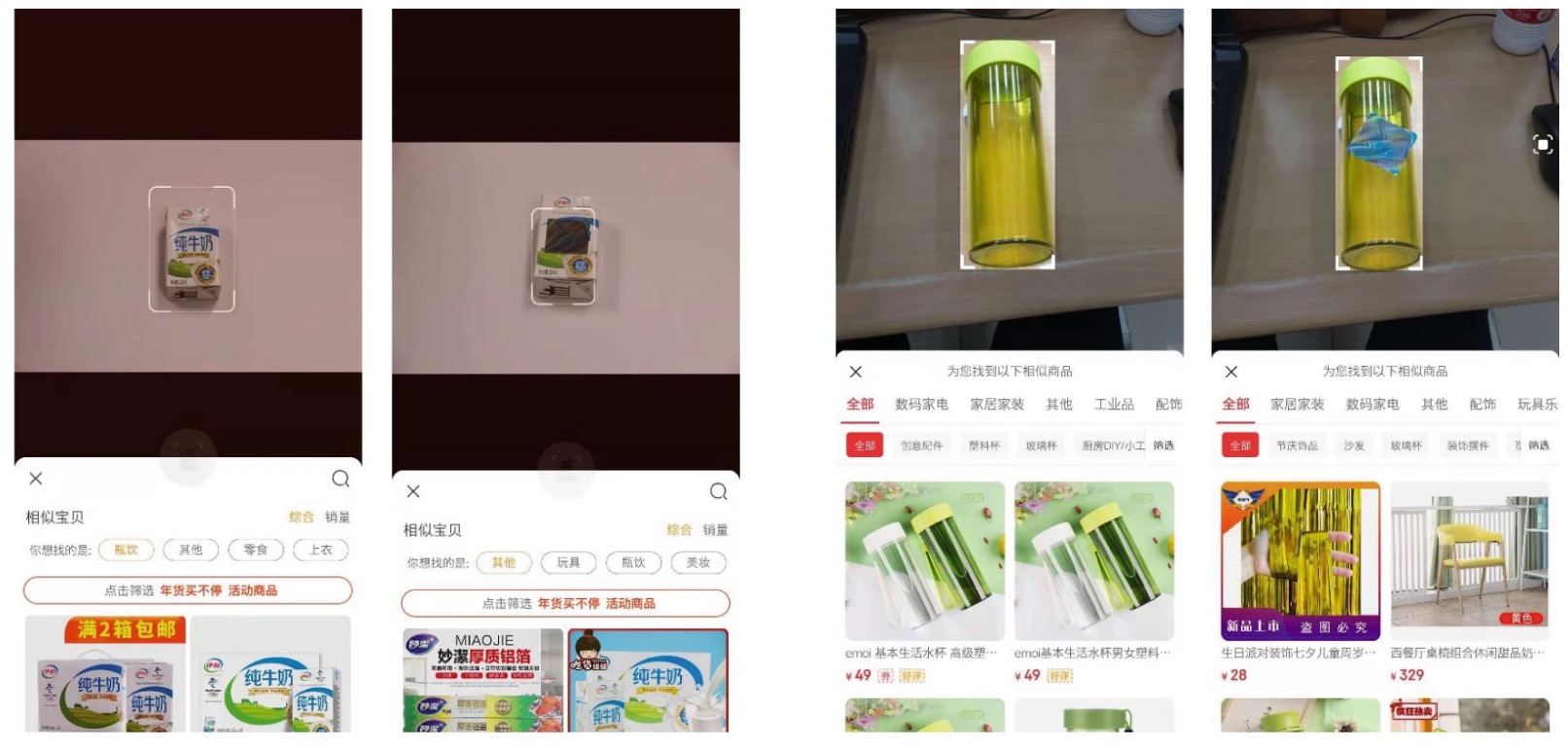}
\end{subfigure}

\captionsetup{font={small}}
\caption{Examples of patch attacks 
\cite{liu2020bias}. 
Two pairs of images each show an item with a patch fooling Taobao.com and JD.com to make the correct prediction ranked second and out of the top 2, respectively, which may distract users from selecting the right items for purchase.}
\label{fig:patch_example}
\end{figure}

%





To our knowledge, only PatchGuard \cite{xiang2021patchguard} and CBN \cite{zhang2020clipped} \textbf{can} provide patch robustness certification guarantees for top-$k$ predictions scalable enough to verify real-life benchmarks like ImageNet
in the literature
\cite{zhang2020clipped,xiang2021patchguard,xiang2022patchcleanser,saha2023revisiting,salman2022certified,chiang2020certified,metzen2021efficient, chen2022towards,zhou2023majority, patchcensor, li2022vip, mccoyd2020minority}.
They are both \emph{voting-based} certified recovery, a 
mainstream certified defense approach for patch attacks 
\cite{zhang2020clipped,xiang2021patchguard,levine2020randomized,salman2022certified,metzen2021efficient,chen2022towards,zhou2023majority,li2022vip}. 
However, our findings (see Section 
2.4) further shows that, as the patch size increases linearly,
both defenders 
certify ImageNet samples with increasing difficulties, such as becoming unable to certify \emph{any} ImageNet samples if the patch size reaches 96 pixels (about a quarter of a sample's area) unless $k = 1000$ (note that there are only 1000 classes for ImageNet, therefore, the guarantee is trivial for all ImageNet samples).
We find that their pair-wise comparison inflates the attack budget (the number of votes an attacker can control) during the certification (see Section 3.1).
A straightforward solution without inflation is to check against all the possible combinations (the number is astronomical) of allocation of the attack budget, which is intractable.
Masking-based certified recovery, like PatchCleanser \cite{xiang2022patchcleanser}, is another mainstream approach but \textbf{cannot} give certifications for top-$k$ predictions 
because its certification theory heavily relies on the full consensus of mutants (all votes to the true label),
only support certification with the true label as the top-1 label.
Interval bound propagation, such as IBP \cite{chiang2020certified}, 
performs 
symbolic value propagation across layers, which can only certify low-resolution samples 
\cite{xiang2021patchguard}.

This paper proposes CostCert, a novel, scalable, and precise defender for voting-based certified recovery for top-$k$ predictions.
Our main insight is that those votes unaffected by attackers hold the key to determine the smallest total additional votes (smallest tie cost) atop them 
once the required $k$ is given.
By comparing this cost with the attack budget, CostCert can provide the guarantee precisely without the inflation problem and combinatorial enumeration, the \emph{first work} to achieve both.
A  voting-based recovery defender hides different parts of a sample \textit{x} to generate its ablated mutants (see Fig.~\ref{fig:voting_2}).
The prediction label of 
a mutant that overlaps with a given patch region is deemed to be controlled by attackers.
The defender counts the prediction label of each mutant as a vote for that label, and outputs the labels in the descending order of the votes they receive as the output $g(x)$.
For each pair of a label $y \in g(x)$ and a patch region,
unlike all existing works,
CostCert only counts the number of votes received by $y$ that attackers \emph{cannot} control, which we refer to it
as the \textbf{\textit{clean vote}} of $y$.
Suppose that there are \emph{$n$}  ($n \leq k$) labels not smaller than $y_0$ in the clean vote.
For each patch region, 
CostCert finds $k'$ $(k'=k-n+1)$ labels from $g(x)$ each with a smaller clean vote than $y_0$ such that they as a whole are allocated with minimum additional votes atop their clean votes to make their individual resulting numbers of votes not smaller than $y_0$'s clean vote, attaining the minimum number of additional votes in total (called the smallest tie cost), 
i.e.,
to ensure $k$ number of labels not to rank after 
$y_0$ at minimum.
If the attack budget (i.e., {the maximum number} of votes controlled by attackers) is smaller than 
{the smallest tie cost for all patch regions,
CostCert reports \textit{x} as $k$-certified, i.e., the true label $y_0$ must remain one of the top-$k$ labels even if a malicious version of \textit{x} is input to the defender for top-$k$ predictions, proven by} Thm.~\ref{Thm:new_tk}.
The experiment shows that CostCert wins the peers in top-$k$ certification to a larger extent as $k$ increases or as the patch size increases for the same $k$.

The main contribution of this paper is threefold. (1) The paper proposes CostCert. 
(2) It formally proves CostCert's soundness and its superior precision. 
(3) It validates CostCert's 
high effectiveness and scalability through experiments.



The rest of the paper is organized as follows. We first revisit the preliminaries of this work in \S 2. Then, \S 3 presents our proposal CostCert followed by its evaluation in \S 4. \S5 reviews the related works, and we conclude this paper in \S 6.

\section{Preliminary}
\label{sec:pre}

\subsection{Classification model}
\label{sec:classification-model}
We represent an image sample $\textit{x}$ as a matrix of pixels with $h$ rows and $w$ columns taken from 
an input space $\mathcal{X} \subset \mathbb{R}^{w \times h}$.
A classifier $g$ takes $\textit{x}$ as input to output a sequence of all labels $g(\textit{x})$ {ranked} by the confidence,
where
each label is taken from a 
 label space $\mathcal{Y}=\{0, 1, \cdots, |\mathcal{Y}|-1\}$, and
we follow the existing work of voting-based certified recovery defenders 
\cite{xiang2021patchguard, zhang2020clipped} 
to use the number of ablated mutants  of \textit{x} voting for a label $y$ as $y$'s (vote) confidence  $v_y(x)$ (see Fig. 2 and Section 
2.3.1) for a voting-based recovery defender,
i.e., $g(\textit{x})=\langle y_1, \ldots, y_{|\mathcal{Y}|}\rangle$, where $v_{y_i}(\textit{x})\geq v_{y_j}(\textit{x})$ if $i<j$ (tie cases are resolved following \cite{levine2020randomized} in practice, which conservatively remain unknown for certification). 
We use $g_i(\textit{x})$
to denote 
the label placed in the $i$-th position in $g(\textit{x})$ and $g^k(\textit{x})=\langle y_1, \ldots, y_{k}\rangle$
to denote 
the prefix of $g(\textit{x})$ 
with length $k$ where $1 \leq k \leq |\mathcal{Y}|$.
{For the top-1 and top-$k$ predictions, the top-1 label $g_1(\textit{x})$ and the sequence of the top-$k$ labels $g^k(\textit{x})$ are output, respectively.}


\begin{figure}[b] 
\centering
\includegraphics[width=1\linewidth]{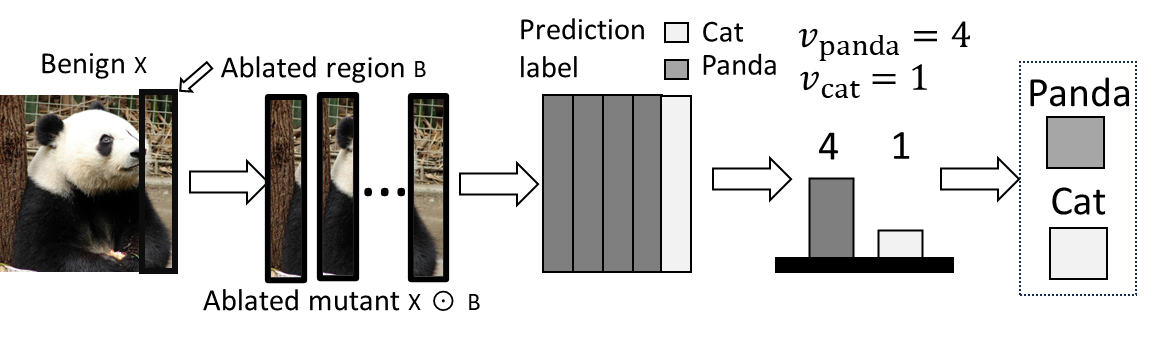}
\captionsetup{font={small}}
\caption{Voting-based certified recovery.
Such a defender $R$ makes predictions in three steps.
Step (1): $R$ synthesizes the image \textit{x} with all 5 ablated regions, which
produces 5 ablated mutants (and 3 mutants are shown as three stripes).
Step (2): 4 and 1 mutants vote for 
\textit{Panda} (dark) and \textit{Cat} (light), and their sums of votes (vote confidence) are denoted by $v_{\text{panda}}$ and
$v_{\text{cat}}$.
Step (3): $R$ outputs a sequence of labels in descending order of vote confidence, which is $\langle$
\textit{Panda},  \textit{Cat}%
$\rangle$.
}
\label{fig:voting_2}
\end{figure}

\subsection{Patch attack}
A patch attacker can modify pixels within a specific \textbf{patch region} of an image, denoted by a matrix $\textsc{p} \in \mathbb{P} \subset [0,1]^{w \times h}$, where $\mathbb{P}$ represents the set of all possible patch regions. The elements within the patch region are set to 1, otherwise 0. 
Following \cite{patchcensor,xiang2021patchguard,xiang2022patchcleanser,zhou2024crosscert}, {we model the samples generated from a patch attacker by an attacker constraint set} $\mathbb{A}(\textit{x})$, defined as $\mathbb{A}(\textit{x}) = \{\textit{x}' \mid \textit{x}'=(\textsc{J}-\textsc{p})\odot \textit{x}+\textsc{p}\odot \textit{x}'' \land \textsc{p}\in \mathbb{P} \}$, where \textsc{J} is an all-one matrix \cite{horn2012matrix}, $\textit{x} \in \mathcal{X}$ is the original image without modification (\textbf{benign sample}), $\textit{x}' \in \mathcal{X}$ is an image provided by the attacker (\textbf{malicious sample}), and $\textit{x}'' \in \mathcal{X}$ is an arbitrary image. $+$, $-$, and $\odot$ are element-wise matrix operators for addition, subtraction, and multiplication \cite{zhou2024crosscert}.
Intuitively, this constraint set for \textit{x} quantifies that any change to \textit{x} can be arbitrary and must be made within exactly one patch region 
(i.e., $\textsc{p}\odot \textit{x}''$) and any other part of the resulting malicious sample should be the same as \textit{x} (i.e., $(\textsc{J}-\textsc{p})\odot \textit{x}$).
Attackers can apply any technique to produce $\textit{x}'$. 
 Given a benign sample $\textit{x}$ and its true label $y_0$, an (untarget) attacker aims to find a sample $\textit{x}' \in \mathbb{A}(\textit{x})$ such that $y_0 \notin g^k(\textit{x}')$
 \cite{zhang2022investigating, tursynbek2022geometry}.


\subsection{Voting-based Certified recovery against patch attacks}
\subsubsection{Definitions and defenders}
\label{sec:voting-based-recovery}






A \textbf{voting-based recovery defender} \cite{zhang2020clipped,xiang2021patchguard,levine2020randomized,salman2022certified,metzen2021efficient,chen2022towards,zhou2023majority,li2022vip} $R = \langle g, c \rangle$
includes a classifier $g$ and a certification analyzer\,$c$:
Given a sample $x$,  $g$ predicts $g(x)$ and $c$ decides whether $R$ certifies $x$.
$R$
aims at $g$ predicting the true label of a certified sample, even under patch attacks.
Its main idea is to limit the impact of any potential patch on the classifier $g$ built by the base classification model $f$
by cutting $x$ into pieces (called ablated mutants) to serve as $f$'s inputs and count their top-1 labels ($\in \mathcal{Y}$) to represent their vote confidences for the labels in $g(\textit{x})$ \cite{levine2020randomized,xiang2021patchguard,salman2022certified,li2022vip, chen2022towards, metzen2021efficient, zhang2020clipped, zhou2023majority}.

We represent an \textbf{ablated region} by a matrix $\textsc{b}\in\mathbb{B}\subset [0,1]^{w \times h}$, where $\mathbb{B}$ contains all ablated regions.
The elements within and outside the region are respectively set to 1 and 0 in $\textsc{b}$.
Applying \textsc{b} to \textit{x} (i.e., $\textit{x}_\textsc{b}$=$\textit{x}\odot\textsc{b}$) generates the \textbf{(ablated) mutant} $\textit{x}_\textsc{b}$ of \textit{x}.
The top-1 prediction label $y$ made by $f$ with $\textit{x}_\textsc{b}$ as input (denoted by $f_1(\textit{x}_\textsc{b})$) is counted as one \textbf{vote} for $y$.
The number of votes received by $y$ from all such mutants of \textit{x} represents 
the \textbf{vote confidence} for $y$, defined as $\mathbf{v_y(\textit{x})}=|\{\textsc{b}\in \mathbb{B}\mid f_1(\textit{x}_\textsc{b})=y\land \textit{x}_\textsc{b}=\textit{x}\odot\textsc{b}\}|$.
Fig.~\ref{fig:voting_2} illustrates the common prediction mechanism in voting-based defenders \cite{levine2020randomized,salman2022certified,li2022vip, chen2022towards, metzen2021efficient, zhang2020clipped, zhou2023majority}.




Def. \ref{def:cert_recovery_t1} captures the concept of certified recovery for the top-1/top-$k$ predictions.
A limitation of certified recovery for the top-1 prediction scenario is that the certified accuracy may be (inherently) low for some tasks. E.g.,  in PatchGuard's experiment \cite{xiang2021patchguard}, only 47.6\% 
of all these ImageNet samples whose top-1 predictions are correct can be certified by PatchGuard, leaving an estimate of 52.4\% in the input space open for attack.

\begin{Definition}[Certified Recovery for Top-1/Top-$k$ Predictions]\label{def:cert_recovery_t1}
\label{def:cert_recovery_tk}
A defender $R = \langle g, c\rangle$ certifies a benign sample $\textit{x}$ with true label $y_0$ as \textbf{\textit{k}-certified samples
}
if the conditions  
$[\forall \textit{x}' \in \mathbb{A}(\textit{x}),y_0\in g^k(\textit{x}') \land
y_0\in g^k(\textit{x}) 
]$
hold.
It is \textbf{1-certified} if $k=1$.
\end{Definition}

\begin{figure}[b] 
\centering
\includegraphics[width=\linewidth]{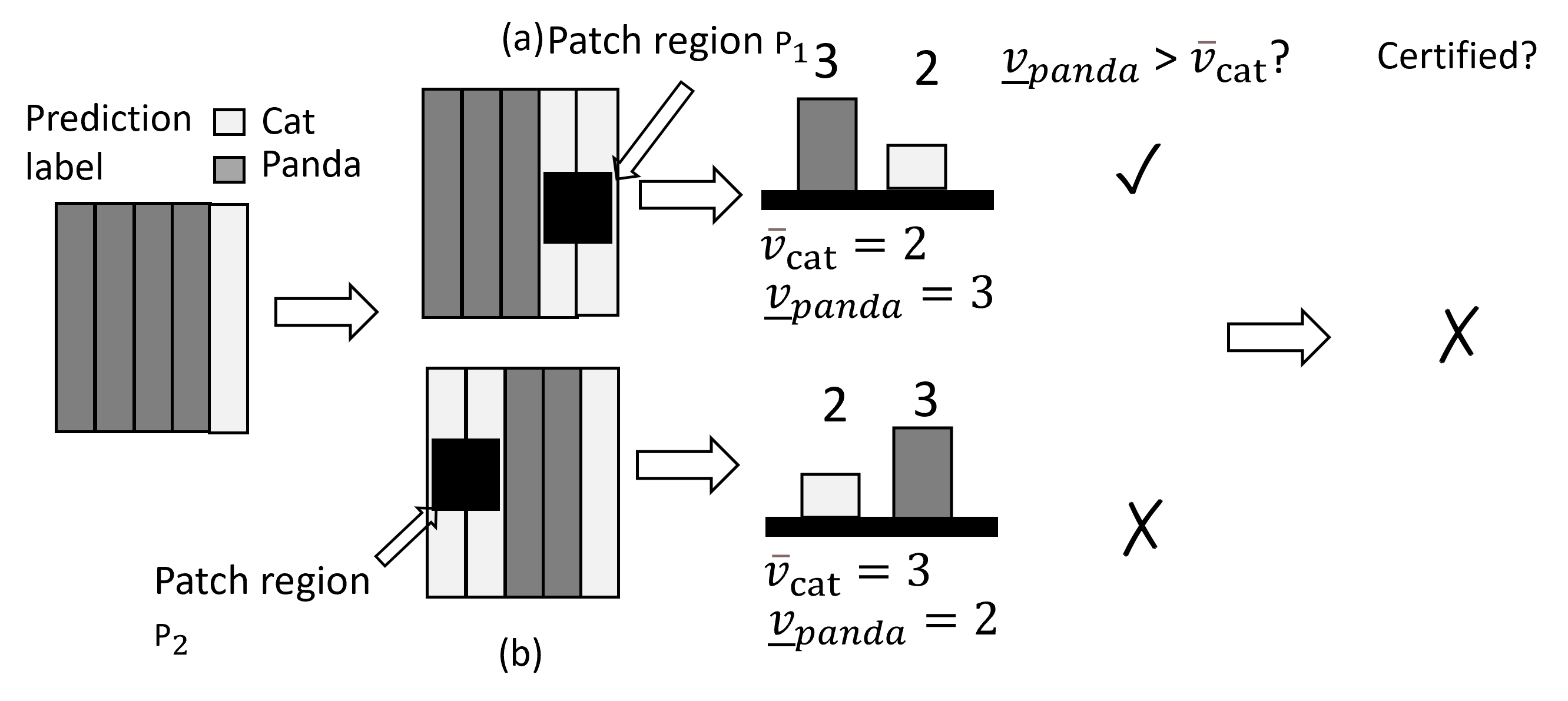}
\captionsetup{font={small}}
\caption{Certification for the 
top-1 prediction by Strategy 2. The set of 5 ablated mutants with vote confidence $v_\text{panda}$ = 4 and $v_\text{cat}$ = 1 is checked against every patch region.
In the upper scenario, the patch region $\textsc{p}_1$ overlaps with the rightmost two mutants.
It can change their labels to \textit{Cat} in the worst case. Since $\underline{v}^{\textsc{p}_1}_\text{panda} = 3 > 2 = \overline{v}^{\textsc{p}_1}_\text{cat}$, the check is passed. In the lower scenario, the patch region $\textsc{p}_2$ overlaps with the leftmost two mutants: $\underline{v}^{\textsc{p}_2}_\text{panda} = 2 \ngtr 3 = \overline{v}^{\textsc{p}_2}_\text{cat}$, which fails the check.
As not all checks are passed, the sample cannot be certified.
}
\label{fig:voting-pg}
\end{figure}


\subsubsection{Robustness certification for top-1 prediction}
\label{sec:top-1}
All existing voting-based defenders \cite{levine2020randomized,salman2022certified,li2022vip, chen2022towards,xiang2021patchguard,zhang2020clipped,zhou2023majority,metzen2021efficient} use the idea of lower/upper bounds for individual labels to compare such bounds for individual label pairs in their certification analyzers, by one of the two strategies below.

%

Let $\Delta$ 
be the maximum number of ablation regions overlapped with a patch region and $y_0$ be the true label of a sample \textit{x}.

Strategy 1 \cite{levine2020randomized,salman2022certified,li2022vip, chen2022towards}
considers $v_{y} - \Delta$ and $v_{y}+\Delta$ as
the lower and upper bounds of the vote confidence $v_{y}$ for each label $y \in \mathcal{Y}$, which
excludes and includes the maximum number of attacker-controlled votes, respectively.
It verifies 
whether the lower bound of the votes for $y_0$ is higher than the upper bound of the votes for each other label $y$ where $y \neq y_0$.
Existing work (e.g.,  \cite{li2022vip}) has proven that \textit{x} is a 1-certified sample if the condition $[{v}_{y_0}>\mathop{\max}_{y\neq y}{v}_{y}+2\Delta]$ holds.
%


Strategy 2
\cite{xiang2021patchguard,zhang2020clipped,zhou2023majority,metzen2021efficient} (see Fig.~\ref{fig:voting-pg}) computes these bounds more precisely than Strategy 1 
\cite{metzen2021efficient}.
%
%
Let $\textsc{O}$ denote the all-zero matrix \cite{zhou2024crosscert}.
Built upon \cite{levine2020randomized,salman2022certified,li2022vip}, it defines the
\textbf{upper bound} $\mathbf{\overline{v}_{y}^\textsc{p}(\textit{x})}$ of the vote confidence ${v}_{y}$ for a label $y$ per patch region \textsc{p} 
as the number of mutants
 that overlap with $\textsc{p}$  \emph{plus} the number of remaining mutants that vote for $y$, i.e.,
 $\overline{v}_{y}^\textsc{p}(x) = |\{
 \textsc{b}\in \mathbb{B} \mid 
\textsc{b}\odot\textsc{p}\neq\textsc{O} \land
\textit{x}_\textsc{b}=\textit{x}\odot\textsc{b}\}|
+ |\{
 \textsc{b}\in \mathbb{B} \mid
 \textsc{b}\odot\textsc{p} = \textsc{O} \land
 f_1(\textit{x}_\textsc{b})=y
\}|.
$
 The \textbf{lower bound}  $\mathbf{\underline{v}_{y}^\textsc{p}(\textit{x})}$ 
 of ${v}_{y}$ is defined as the number of mutants that vote for $y$ and do not overlap with $\textsc{p}$, i.e.,
$\underline{v}_{y}^\textsc{p} = |\{\textsc{b}\in \mathbb{B}\mid\textit{x}_\textsc{b}=\textit{x}\odot\textsc{b}\land \textsc{b}\odot\textsc{p}=\textsc{O}\land f_1(\textit{x}_\textsc{b})=y\}|$.
Prior works (CBN \cite{zhang2020clipped}, PatchGuard \cite{xiang2021patchguard} and BagCert \cite{metzen2021efficient}) prove that 
 \textit{x} is 1-certified
 if the condition  $[\forall \textsc{p}\in\mathbb{P},y' \in \mathcal{Y} \land y'\neq y_0 \land \underline{v}_{y_0}^\textsc{p}>\overline{v}_{y'}^\textsc{p}]$ holds, i.e., for all patch regions, the lower bound for $y_0$ exceeds the respective upper bound for every other label.
 %


PatchGuard \cite{xiang2021patchguard} 
further proposes and applies a masking strategy 
when calculating the vote confidence for each label $y\in\mathcal{Y}$.
Due to the complexity, we invite readers to read the formulation in the original paper \cite{xiang2021patchguard}. 
The general idea of PatchGuard's masking strategy is to find a specific set of consecutive ablated regions (one for the certification phase and another for the operation phase, denoted by $\mathbb{B}^{\text{cer}}_y$ and $\mathbb{B}^{\text{ope}}_y$, respectively)
to discount the votes to each given label $y \in \mathcal{Y}$ by dropping some mutants when computing the votes for each patch region \textsc{p}.
It 
computes  
the vote confidence for $y$ as 
${v}_{y}(x)=|\{\textsc{b}\in \mathbb{B}-\mathbb{B}^{\text{ope}}_y\mid f_1(\textit{x}_\textsc{b})=y\land\textit{x}_\textsc{b}=\textit{x}\odot\textsc{b}\}|$ in the operation phase, and
computes the upper bounds for $y$ with respect to this \textsc{p} as $\overline{v}_{y}^\textsc{p}(x)=|\{\textsc{b}\in \mathbb{B}\mid f_1(\textit{x}_\textsc{b})=y\land\textit{x}_\textsc{b}=\textit{x}\odot\textsc{b}\land\textsc{p}\odot\textsc{b}\neq\textsc{O}\}|$ and the lower bounds as $\underline{v}_{y}^\textsc{p}(x)=|\{\textsc{b}\in \mathbb{B}-\mathbb{B}^{\text{cer}}_y\mid f_1(\textit{x}_\textsc{b})=y\land\textit{x}_\textsc{b}=\textit{x}\odot\textsc{b}\land\textsc{p}\odot\textsc{b}\neq\textsc{O}\}|$ in the certification phase. All other parts are the same as Strategy 2.
With unavoidable drops of some unaffected (benign) mutants, this heuristic strategy further limits the effect of the adversarial patch to change the votes of mutants.
For instance,  
a 1.5\% gain in certified top-1 accuracy for 
PatchGuard using DRS \cite{levine2020randomized} (PG-DS \cite{xiang2021patchguard}),
is observed on CIFAR10 
if a ``proper'' patch region size can be used in the operation phase. 
However, other authors \cite{salman2022certified} have commented that the above guarantee from such a strategy is ``\emph{brittle}''. It is challenging to determine a ``proper'' patch region size for the operation phase. 


\subsubsection{Robustness certification for top-\textit{k} predictions}\label{sec: pre_top_k}


 

 PatchGuard \cite{xiang2021patchguard} and CBN \cite{zhang2020clipped} have generalized their certification analyses to certify $k$-certified samples.
 Specifically, Strategy 2 is slightly modified to 
 verify whether 
 $\mathbf{\overline{v}_{y'}^\textsc{p}(\textit{x})}$ 
for $y'$  is \emph{inclusively} greater than 
$\underline{v}_{y_0}^\textsc{p}(\textit{x})$ 
for $y_0$ for each \textsc{p}.
Let the number of such labels satisfying this modified condition be $k^\textsc{p}$.
If $k^\textsc{p} < k$ for some common constant $k$ across all $\textsc{p} \in \mathbb{P}$, \textbf{\textit{x} is \textit{k}-certified} by such a defender.
 Thm. \ref{thm:old_tk} from \cite{zhang2020clipped} captures this result.
Note that despite of PatchGuard 
using a masking strategy, the certification analyzers of PatchGuard and CBN
both are built on Thm. \ref{thm:old_tk}, i.e., their certification analyzers $c$ are commonly defined as $c(\textit{x},k)=[\forall \textsc{p}\in\mathbb{P},|\{y' \in \mathcal{Y} \mid y' \neq y_0 \land \overline{v}_{y'}^\textsc{p}(\textit{x})\geq \underline{v}_{y_0}^\textsc{p}(\textit{x})   \}| < k]$.
Fig. \ref{fig:voting-topk-old} depicts their analyses.






\begin{thm}[Certified recovery for top-$k$ predictions \cite{zhang2020clipped}]\label{thm:old_tk}
Suppose $[\forall \textsc{p}\in\mathbb{P},|\{y' \in \mathcal{Y} \mid y' \neq y_0 \land \overline{v}_{y'}^\textsc{p}(\textit{x})\geq \underline{v}_{y_0}^\textsc{p}(\textit{x})   \}| < k]$ holds. 
Then, $[\forall \textit{x}'\in\mathbb{A}(\textit{x}), y_0\in g^k(\textit{x}')]$ (see Def.~\ref{def:cert_recovery_t1}) holds.

\end{thm}



\begin{figure*}[th] 
\centering
\includegraphics[width=1\linewidth]{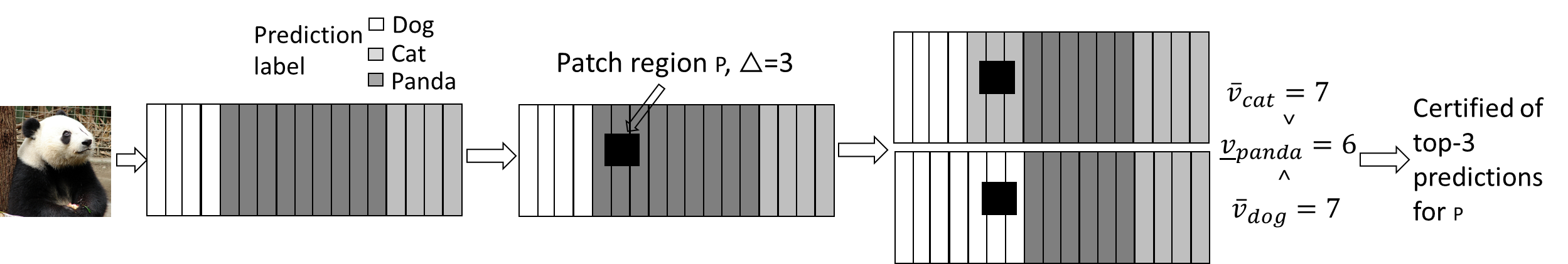}
\captionsetup{font={small}}
\caption{Illustration of PG$_\spadesuit$/CBN \cite{zhang2020clipped} for top-$k$ predictions to certify 
a sample \textit{x} with the true label  \textit{Panda}. The defender produces 17 ablated mutants for \textit{x} (4, 4, and 9 mutants predicted to \textit{Cat}, \textit{Dog}, and \textit{Panda}, respectively. The differences between the true label and other labels is noticeable.) 
followed by analyzing with each patch region \textsc{p} (black square), where the one overlapping with 3  mutants predicted to \textit{Panda} is shown.
An attacker may alter the labels of these 3 mutants to all \textit{Dog} (the upper scenario), all \textit{Cat} (the lower scenario), or some other combinations.
The two illustrated scenarios result in 
$\underline{v}^{\textsc{p}}_\text{panda}(\textit{x}) = 6 < 7 = \overline{v}^{\textsc{p}}_\text{dog}(\textit{x})$ and
$\underline{v}^{\textsc{p}}_\text{panda}(\textit{x}) = 6 < 7 = \overline{v}^{\textsc{p}}_\text{cat}(\textit{x})$, respectively.
(Note that all other scenarios cannot produce vote confidence for \textit{Panda} higher than 6.)
The label \textit{Panda},  originally at the first position, in the worst case, drops to the third position according to Thm. \ref{thm:old_tk}.
Thus, the defender certifies \textit{x} as $k$-certified with $k =3$ after checking all patch regions. 
PatchGuard \cite{xiang2021patchguard} also certifies \textit{x} as $k$-certified with $k =3$ 
because the upper bounds of vote confidence for \textit{Cat} and \textit{Dog} are still both higher than the lower bound of vote confidence for \textit{Panda} after applying the masking strategy.
\textbf{Limitation:} At a closer look, in the scenario of $\overline{v}^{\textsc{p}}_\text{dog}(\textit{x}) = 7$,  the labels for the three mutants that overlaps with \textsc{p} cannot simultaneously be predicted to multiple labels: \textit{Cat} and \textit{Dog}.
It indicates that the scenarios for $\overline{v}^{\textsc{p}}_\text{cat}(\textit{x}) = 7$ and $\overline{v}^{\textsc{p}}_\text{dog}(\textit{x}) = 7$ mentioned above \textbf{cannot co-exist simultaneously}.
It implies that, ideally, we should have $\overline{v}^{\textsc{p}}_\text{cat}(\textit{x}) = 4$ instead of 7 if $\overline{v}^{\textsc{p}}_\text{dog}(\textit{x}) = 7$, and vice versa.
}
\label{fig:voting-topk-old}
\end{figure*}



\subsection{Motivation}\label{sec:Motivational}
Our work is motivated by using PatchGuard \cite{xiang2021patchguard} to certify the ImageNet validation set \cite{deng2009imagenet} for top-$k$ predictions (among 1000 classes).
We implement PatchGuard (referred to as PG) following PG-DS \cite{xiang2021patchguard} (which adopts Strategy 2 with the masking strategy) and replace the original RestNet base model with Vision Transformer (ViT, version ViT-B16-224) \cite{dosovitskiy2021an} to align with the practice in \cite{salman2022certified,li2022vip,zhou2024crosscert}. We ablate the masking strategy from PG to construct an ablated defender PG$_\spadesuit$, which makes
PG$_\spadesuit$ shares the same fundamental certification theory with CBN \cite{zhang2020clipped}.
Note that CBN is customized on BagNet, making it hard to directly replace its base model by ViT.
%
The experimental setup is the same as our main experiment (see Section 
4). Similar to \cite{patchcensor,li2022vip}, we vary the patch (region) size \cite{patchcensor,li2022vip} by increasing it from 16 to 112 pixels with a step size of 16 pixels and denote the series by $\langle m_i\mid i=1,2,...\rangle$ (e.g., $m_1 = 16$ and $m_2 = 32$).
We  collect the smallest  $k$  sufficient to certify each sample as $k$-certified (referred to as min$k$).

Fig.~\ref{fig:motivate} summarizes the results.
The average values of min$k$ for PG and PG$_\spadesuit$ are surprisingly high: their means are as high as min$k$ = {199} for PG and min$k$ = {333} for PG$_\spadesuit$ when $m_i$ = {16} (a small value), which acceleratedly increases as $i$ for $m_i$ increases and reaches 1000 (i.e., the guarantee becomes trivial) when $m_i$ reaches 96, where the trend of acceleration can be seen clearly from the plot on the right. 
Their median min$k$ values are 1 when  $m_i=32$ and grow quickly to become 1000 when $m_i = 48$ for PG$_\spadesuit$
and $m_i = 80$ for PG.
Meanwhile, in both defenders, 
the patch region size is linear to $\Delta$, and Strategy 2 also compares the lower/upper bounds with $\Delta$ linearly.
%
 It makes us wonder the following question. 
\vspace{-0.5ex}
\begin{center}
\fbox{
\begin{minipage}[t]{0.9\linewidth}
What are the fundamental limitations in existing defenders offering certified recovery protection for top-$k$ predictions?
\end{minipage}
}
\end{center} 
\vspace{-1ex}






\begin{figure}[t] 
\centering
\begin{subfigure}[t]{0.48\linewidth}
\includegraphics[width=\linewidth]{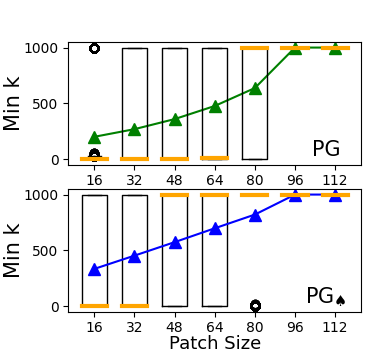}
\captionsetup{font={small}}
\end{subfigure}
\begin{subfigure}[t]{0.48\linewidth}
\centering
\includegraphics[width=\linewidth]{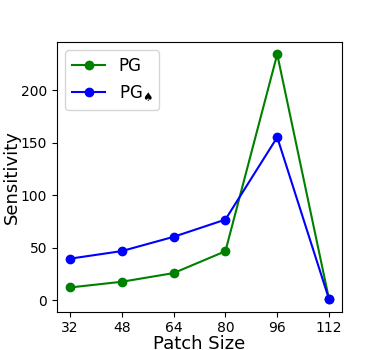}
\captionsetup{font={small}}
\end{subfigure}
\captionsetup{font={small}}
\caption{The two boxplots for PG and PG$_\spadesuit$ on certifying the 
ImageNet 
test dataset 
are shown on the left. 
The $y$-axis is min$k$. 
The $x$-axis is the patch size $m_i$.
The markers
stand for 
the mean values.
The chart on the right shows the sensitivity of PG and PG$_\spadesuit$ on ImageNet by varying the path size. 
Let $k_{m_i}$ be the mean min$k$ for size $m_i$ in the boxplots on the left. The $y$-axis is the sensitivity ratio $r_i$, defined as $\nicefrac{k_{m_{i}}}{k_{m_{i-1}}}$ and the $x$-axis is $m_i$.
The chart shows the acceleration of the mean min$k$ for each defender increases as $m_i$ increases. 
The min$k$ for every sample reaches 1000 (the total number of classes) when $m_i$ is 96 or 112, and $r_i$ drops to 1 when  $m_i$ is 112.}\label{fig:motivate}
\end{figure}

\section{CostCert}

This section presents CostCert
(see Fig. \ref{fig:voting} for overview).

\begin{figure*}[t] 
\centering
\includegraphics[width=1\linewidth]{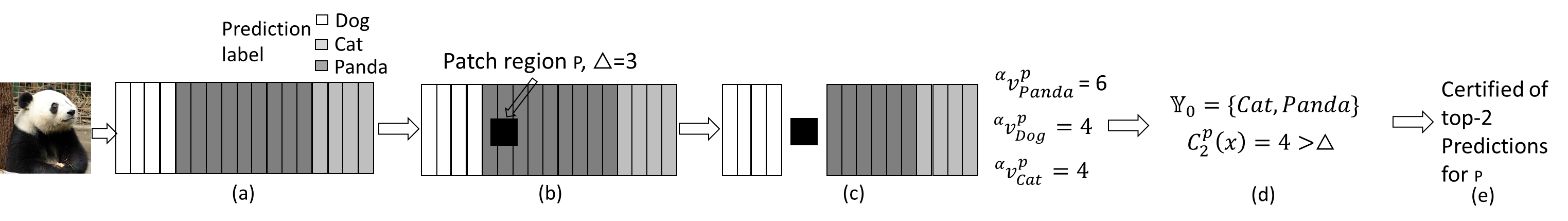}
\captionsetup{font={small}}
\caption{Illustration of CostCert. The diagram contains five sections, separated by arrows, from left to right, numbered as (a)--(e).
(a) CostCert produces 17 ablated mutants from the sample \textit{x} with the true label \textit{Panda}, 
where a patch region can overlap with at most three mutants (i.e., $\Delta$ = 3) --- 4, 9, and 4 mutants are predicted to the labels \textit{Dog}, \textit{Panda}, and \textit{Cat}, respectively.
The prediction output is $\langle \textit{Panda}, \textit{Dog}, \textit{Cat} \rangle$ sorted in descending order of vote confidence. Note the tie between \textit{Dog} and \textit{Cat} is resolved randomly.
(b) CostCert finds the three mutants overlapped with the patch region \textsc{p}. 
(c) CostCert removes these mutants regardless of their labels to compute the clean votes ($\alpha$-votes) at the clean level {$\alpha^\textsc{p}(\textit{x})$} for every label: 6, 4, and 4 for \textit{Panda}, \textit{Dog}, \textit{Cat}, respectively.
(d) CostCert computes the cost $C^\textsc{\footnotesize p}_2(\textit{x})$ required to move \textit{Panda} down from the first position to the third position in the prediction, which is 4, according to Eq.~\ref{eq:cost}, and finds $4 > 3 = \Delta$, which means that the label \textit{Panda} cannot be pushed to the third position with the available attack budget and keeps in top-2. 
(e) After checking all patch regions through the procedures illustrated in (b)--(d), CostCert finds that the condition $[\forall p \in \mathbb{P}, C^\textsc{\footnotesize p}_2(\textit{x}) > \Delta]$ holds.
Therefore, by Thm.~\ref{Thm:new_tk}, it certifies \textit{x} as $k$-certified with $k =2$. 
}
\label{fig:voting}
\end{figure*}

\subsection{Problem and Challenge}\label{sec:insight}
One of the most challenging tasks in robustness certification is that 
it is impractical to check
 all malicious samples around the benign ones.
 Certified defenders 
against patch attacks find common characteristics of all these malicious samples to certify their benign versions. 
If the maximum size of patches is given, voting-based certified defenders can calculate the number of mutants that can be affected by the attacker through geometric relations (Strategy 1 \cite{levine2020randomized,salman2022certified,li2022vip} 
or Strategy 2 \cite{xiang2021patchguard,zhang2020clipped,zhou2023majority,metzen2021efficient}) to perform their certification analyses.
Let us adopt the same practice to suppose $\Delta$ to be the maximum number of ablation regions (mutants) that can be overlapped by a patch region.

\textbf{Problem:}
A patch region that overlaps $q$ mutants (where $q \leq \Delta$) can only affect a total of $q$ votes among all labels because any mutant can only cast one vote no matter which labels this vote is assigned to. In the worst case, the patch region can overlap with $\Delta$ mutants. Thus, an attacker can only have at most $\Delta$ atomic votes as the \textbf{attack budget} to adjust the index positions of all labels in $g(x)$.
In contrast, for the 1000-class ImageNet classification model, for example, Strategies 1 and 2 presented in Section
2.3.2 assume an excessive attack budget of 999$\Delta$ votes. 
In Fig. \ref{fig:voting-topk-old}, existing certification analyzers conservatively assume that an attacker has an attack budget of 6 votes despite of $\Delta=3$ on a three-class classification task so as to assign 3 votes to \textit{Dog} and 3 votes to \textit{Cat}. A separate comparison between the lower and upper bounds of vote confidence for any two labels with the same attack budget will overestimate the overall power of the attack on the sample under certification.
Moreover, increasing the value of $\Delta$ amplifies the problem further.

\textbf{Challenge:}
To avoid the inflation of the attack budget, one straightforward solution is to check all possible cases of attacks to change the allocation of the votes under this attack budget $\Delta$.
Although we also know the attacker would not allocate the votes to the true label ($\Delta$ votes are changed to allocate to  $|\mathcal{Y}| - 1$ labels, and to simplify and reduce the cases, we regard these $\Delta$ votes are indistinguishable), 
there are still \emph{too many possible combinations} ($\binom{\Delta + |\mathcal{Y}| - 2}{|\mathcal{Y}| - 2}$ according to the ``balls and bars'' model in combinatorics \cite{flajolet2009analytic}) of the vote allocation for the attacker.
For example, in ImageNet when the patch size is 32 pixels (2\%), $\Delta=50$, $|\mathcal{Y}|=1000$,
where exhaustively enumerating all possible cases for the attacker to allocate the controlled votes is intractable.




\subsection{Insight and Its Rationale}
\label{sec:overview-CostCert}

Mutants casting their votes is a fundamental  concept in voting-based recovery. 
The number of votes received by a label with respect to a patch region \textsc{p} are the votes from these mutants that do not overlap (unaffected) and overlap (affected) with \textsc{p}, referred to as the \emph{clean vote} ($\alpha$-vote) and \emph{dirty vote} ($\beta$-vote) of the label, respectively. They add up to become the \emph{observed vote} of the label in the prediction output for any sample.

\textbf{Insight:}
Clean vote is the key.
When the required $k$ is given, by focusing on those clean votes (unaffected by the attacker), we can determine the minimum cost on top of them to push the true label out of the top-$k$ labels.
By comparing this minimum cost with the attack budget, CostCert is able to provide certification, which eliminates the need to enumerate all pairwise comparisons and check all combinations.

Specifically, we study the problem through vote allocation.
We can collect the clean votes of all labels for a benign sample $x$ with respect to a given patch region into a set and then
split the set into two subsets: one contains all these clean votes not smaller than the clean vote of the true label, and the other set contains the remaining clean votes, denoted by sets $D$
and $E$, respectively.
To push the true label further down by $k-n+1$ index positions when a patched version of $x$ is given, where $n = |D|$ and $k$ is a number given by a user (i.e., the true label will not be one of the top $k$ labels after a patch attack), we can further split $E$ into two subsets, denoted by sets $E_1$ and $E_2$, with $E_1$ containing $k-n+1$ number of clean votes.
To make the true label rank after all these labels whose votes are in either $D$ or $E_1$, at minimum, we need to increase the observed votes of all these labels whose votes are in $E_1$ to tie with the clean vote of the true label. 
In this way, we will have at least $k+1$ labels (including the true label) with votes not smaller than the clean vote of the true label, and 
thus can determine the number of additional votes required (referred to as the \emph{the tie cost}).
Since the clean vote of the true label is known, if $E_1$ contains the greatest sum of clean votes among all possible splits from $E$, the number of votes needed for the above increase in the observed votes will be the smallest for the patch region (referred to as the \textbf{\textit{smallest tie cost}}).

Computing $E_1$ with this greatest sum of clean votes, where we denote the resulting $E_1$ by $\mathbb{Y}_0$ and this greatest sum by $C_k^\textsc{\footnotesize p}(\textit{x})$, is efficient due to our design because clean votes of all labels are invariant across all possible patches of the same patch region, resulting in all labels being arranged in the same total order for the same patch region \emph{without over-approximation of the upper bounds of votes or double counting of dirty votes due to the variation of observed votes among possible patches of the same patch region.}%
\footnote{
In contrast, if observed votes are used instead, each vote in the sets corresponding to sets $D$ to $E_1$ will not be a number, rather, it would be a pair of upper bound and lower bound for each label and the total order property CostCert had achieved is degenerated into a weaker partial order, making the splitting of these sets result in many possible combinations of those subsets and many possible partial order relations among labels, failing existing defenders to apply our novel strategy with those above-mentioned problems if they aim to find the smallest cost.
}

Therefore, if the above smallest tie cost is strictly larger than $\Delta$ (the maximum number of votes controlled by  attackers) for all patch regions and a given $k$ value, i.e., check whether $[C_k^\textsc{p}(\textit{x}) > \Delta]$ for all \textsc{p} holds, the sample is $k$-certified.

We will illustrate CostCert after we formally introduce clean vote and dirty vote in the next section.

\subsection{Formal Definitions of Clean Vote and Dirty Vote}

Suppose $R = \langle g, c\rangle $ is a voting-based recovery defender with $f$ as the base model of $g$, and \textit{x} with its true label $y_0$ is a sample under certification. 
Like prior work \cite{salman2022certified,li2022vip},
to take a conservative view, 
CostCert assumes that if $y_0$ and some other labels receive the same vote confidence, it always places $y_0$ right after all these labels in $g(.)$ \textit{in the certification analysis} since the tie-breaking strategy of $g(.)$ may be unknown to CostCert.
We refer to the set of clean votes for all labels $\mathcal{Y}$ with respect to \textsc{p} as the {clean level} for \textsc{p}, denoted by $\alpha^\textsc{p}(\textit{x})$, 
and defined as 
$\alpha^\textsc{p}(\textit{x}) = \{
\sideset{^\alpha}{^\textsc{p}_{y}}{\mathop{v}}(\textit{x}) 
\}_{y \in \mathcal{Y}}$ where 
the \textbf{clean vote ($\alpha$-vote)} for $y$ is $\sideset{^\alpha}{^\textsc{p}_{y}}{\mathop{v}}(\textit{x})=|\{\textsc{b}\in \mathbb{B}\mid f_1(\textit{x}_\textsc{b})=y\land\textit{x}_\textsc{b}=\textit{x}\odot\textsc{b}\land \textsc{b}\odot\textsc{p}=\textsc{O}\}|$, where $\textsc{O}$ is an all-zero matrix. 
Similarly, the dirty vote for $y$ is defined as $\sideset{^\beta}{^\textsc{p}_{y}}{\mathop{v}}(\textit{x})=|\{\textsc{b}\in \mathbb{B}\mid f_1(\textit{x}_\textsc{b})=y\land\textit{x}_\textsc{b}=\textit{x}\odot\textsc{b}\land \textsc{b}\odot\textsc{p}\neq\textsc{O}\}|$.
(We note that
$^\beta{v}_y^\textsc{p}(.)$ is only a notional concept to ease readers to follow the overview of CostCert and ease our presentation on the proof of Thm.~\ref{Thm:new_tk}.%
)
%

In Fig.~\ref{fig:voting}(c), 
$\alpha^\textsc{p}(\textit{x}) = \{
\sideset{^\alpha}{^\textsc{p}_\textit{Dog}}{\mathop{v}}(\textit{x}),
\sideset{^\alpha}{^\textsc{p}_\textit{Panda}}{\mathop{v}}(\textit{x}),
\sideset{^\alpha}{^\textsc{p}_\textit{Cat}}{\mathop{v}}(\textit{x})
\}$, where 
$\sideset{^\alpha}{^\textsc{p}_\textit{Dog}}{\mathop{v}}(\textit{x}) = 4$,
$\sideset{^\alpha}{^\textsc{p}_\textit{Panda}}{\mathop{v}}(\textit{x}) = 6$, and
$\sideset{^\alpha}{^\textsc{p}_\textit{Cat}}{\mathop{v}}(\textit{x}) = 4$.
%
CostCert checks whether
$\sideset{^\alpha}{^\textsc{p}_\textit{Panda}}{\mathop{v}}(\textit{x}) \times |\mathbb{Y}_0| -(\sideset{^\alpha}{^\textsc{p}_\textit{Dog}}{\mathop{v}}(\textit{x}) + \sideset{^\alpha}{^\textsc{p}_\textit{Cat}}{\mathop{v}}(\textit{x})) 
= 6 \times 2 - (4 + 4) > \Delta = 3 $.

\subsection{Definition, Algorithm, and Theoretical Guarantee}
\label{sec:topkcertProof}
\vspace{-0.5ex}


In Section
3.2, we have presented that CostCert checks the condition $[C_k^\textsc{p}(\textit{x}) > \Delta]$ and finds a set denoted by $\mathbb{Y}_0$ that satisfies the condition $\Phi$. First, we denote $n$ as the conservative index position of the true label $y_0$ ranked by their votes, defined as $n=|\{y\in\mathcal{Y}\mid \sideset{^\alpha}{^\textsc{p}_{y}}{\mathop{v}}(\textit{x}) \geq\sideset{^\alpha}{^\textsc{p}_{y_0}}{\mathop{v}}(\textit{x}) \}|$.
When $n-1\geq k$ (i.e., $y_0$ is not a top-$k$ label), we define $C_k^\textsc{\footnotesize p}(\textit{x})=0$.
When $n-1<k$,
we formalize the term $C_k^\textsc{p}(\textit{x})$ as follows:
\vspace{-0.5ex}
\begin{equation}
\label{eq:cost}
C_k^\textsc{\footnotesize p}(\textit{x})
=(k-n+1) \sideset{^\alpha}{^\textsc{p}_{y_0}}{\mathop{v}}(\textit{x}) 
-
\max_{\mathbb{Y} \subseteq \mathcal{Y} \land 
\Phi(k, y_0, \textsc{p}, \mathbb{Y})}  
\sum_{y\in\mathbb{Y}} \sideset{^\alpha}{^\textsc{p}_{y}}{\mathop{v}}(\textit{x}) 
\end{equation}
where 
$
\Phi(k, y_0, \textsc{p}, \mathbb{Y}) := 
[|\mathbb{Y}|=k-n+1] \land [\forall y\in\mathbb{Y}, \sideset{^\alpha}{^\textsc{p}_{y}}{\mathop{v}}(\textit{x}) <\sideset{^\alpha}{^\textsc{p}_{y_0}}{\mathop{v}}(\textit{x}) ]
$, $k$ is the number of top labels to check whether a given sample \textit{x} subject to robustness certification can be $k$-certified with,
$y_0$ is \textit{x}'s true label, and
\textsc{p} is a patch region.
Also, conceptually, in Eq. \ref{eq:cost}, the scaling factor $k-n+1$ for the $\alpha$-votes of $y_0$ represents $|\mathbb{Y}_0|$, the second subterm of $C_k^\textsc{\footnotesize p}(\textit{x})$ 
computes the sum of $\alpha$-votes for the labels in $\mathbb{Y}_0$, and the $\max(.)$ function with its constraint implements $\Phi$. 

\textbf{CostCert} is $R = \langle g, c \rangle$, where $c(\textit{x},k)= [\forall \textsc{p}\in\mathbb{P}, C_k^\textsc{p}(\textit{x})>\Delta]$,
which returns \textit{True} to indicate that it certifies \textit{x} as $k$-certified, otherwise \textit{False}.
The classifier $g(\textit{x})$ is defined in Section 
2.1.%
\footnote{CostCert can also be implemented upon PG's classifier with the masking strategy. Although such a masking strategy may gain higher certified accuracy, it sacrifices more and more clean accuracy when the patch size gets larger (See Section
4.3.2). We leave the evaluation as a future work.}

Algorithm \ref{alg:top-k-new} presents how CostCert decides \textit{x} to be a $k$-certified sample.
For each patch region \textsc{p}, it computes the $\alpha$-votes for all labels (lines 2--6) and how many labels with not fewer $\alpha$-votes than the true label of $\textit{x}$ (line 7--8),
finds $C^\textsc{p}_k(x)$ (lines 9--11), and verifies the predicate in $c(\textit{x},k)$ for this \textsc{p} (line 12). 
Based on the results on $c(\textit{x},k)$ for all patch regions, 
it either applies Thm. \ref{Thm:new_tk} to infer  \textit{x} to be a $k$-certified sample (line 16) or refrains from drawing this conclusion (line 13).


\begin{algorithm}[t]
\SetAlgoLined
\DontPrintSemicolon
\caption{\text{CostCert} Certification}
\label{alg:top-k-new}
\SetKwInOut{KwIn}{Input}
\SetKwInOut{KwOut}{Output}
\KwIn{Sample $\textit{x}$, true label $y_0$, label space $\mathcal{Y}$, patch region set $\mathbb{P}$, ablation region set $\mathbb{B}$,
certified position of the true label $k$, attack budget 
$\Delta$}
\KwOut{%
Whether $\textit{x}$ is $k$-certified 
}
\ForEach{$\textsc{p}$ $\in$ $\mathbb{P}$
} 
{
$\mathbb{V}^\textsc{p}=\emptyset$
\tcp*{\text{\footnotesize map: label $\rightarrow$ clean vote}}
\ForEach{$y$ $\in$ $\mathcal{Y}$
} 
{
{\footnotesize $\sideset{^\alpha}{^\textsc{p}_{y}}{\mathop{v}}(\textit{x})=$ $|\{\textsc{b}\in \mathbb{B}\mid f_1(\textit{x}_\textsc{b})=y\land\textit{x}_\textsc{b}=\textit{x}\odot\textsc{b}\land \textsc{b}\odot\textsc{p}=\textsc{O}\}|$ 
}
\\
$\mathbb{V}^\textsc{p}[y]=$$\sideset{^\alpha}{^\textsc{p}_{y}}{\mathop{v}}(\textit{x})$
}

$\sideset{^\alpha}{^\textsc{p}_{y_0}}{\mathop{v}}=\mathbb{V}^\textsc{p}[y_0]$
\tcp*{\text{\footnotesize the clean vote of $y_0$}}
$n=|\{y\in\mathcal{Y}\mid \mathbb{V}^\textsc{p}[y]\geq\sideset{^\alpha}{^\textsc{p}_{y_0}}{\mathop{v}}\}|$

$A=(k-n+1)\times \sideset{^\alpha}{^\textsc{p}_{y_0}}{\mathop{v}}(\textit{x})$
\\
$B=
\max_{\mathbb{Y} \subseteq \mathcal{Y} \land 
\Phi(k, y_0, \textsc{p}, \mathbb{Y})} \{
\sum_{y\in\mathbb{Y}} \mathbb{V}^\textsc{p}[y]
\}$ 
\\
 $C^\textsc{p}_k = (n-1 < k? A - B$: 0)
 \tcp*{\text{\footnotesize smallest tie cost}}
 
\If {$C^\textsc{p}_k\le \Delta$ }
{
\Return{$\textit{False}$}
}
}

\Return{$\textit{True}$}
\end{algorithm}




Next, we present the soundness of CostCert following the ideas presented in Section
3.2:
At the clean level $\alpha^\textsc{p}$,
the cost in Eq.~\ref{eq:cost} is the lowest additional votes atop the $\alpha$-votes to make $k$ labels have votes inclusively larger than $y_0$ (Thm.~\ref{thm:thm_new_tk_lowest_cost}). 
Further, if the budget of an attacker is lower than this cost for all patch regions, such an attack is infeasible, which means $y_0$ will remain within the top-$k$ labels after the attack (Thm. \ref{Thm:new_tk}).

\begin{thm}[A discriminant through calculating the \textbf{smallest tie cost} for additional votes]\label{thm:thm_new_tk_lowest_cost}
Suppose that for a given patch region \textsc{p} and a given sample \textit{x},
there are less than $k+1$ labels (including $y_0$) having their $\alpha$ votes larger than or equal to that of the true label $y_0$ (i.e., $n=|\{y\in\mathcal{Y}\mid  \sideset{^\alpha}{^\textsc{p}_{y}}{\mathop{v}}(\textit{x})\geq \sideset{^\alpha}{^\textsc{p}_{y_0}}{\mathop{v}}(\textit{x})\}|<k+1$). 
Suppose an attacker wants to add
votes (aka cost) to some labels $\mathbb{Y} \subset \mathcal{Y}$ (referred to as $\delta$-votes, denoted as $\sideset{^\delta}{_{y}}{\mathop{v}}$ for $y \in \mathcal{Y}, \sideset{^\delta}{_{y}}{\mathop{v}}\in\mathbb{N}_0$) to their respective $\alpha$-votes. 
Note that $\sideset{^\delta}{_{y}}{\mathop{v}}$ is an arbitrary non-negative integer.
We use $\sideset{^\gamma}{_{y}}{\mathop{v}}$ to denote their adjusted votes (i.e., $\sideset{^\gamma}{_{y}}{\mathop{v}}=\sideset{^\alpha}{^\textsc{p}_{y}}{\mathop{v}}(\textit{x})+\sideset{^\delta}{_{y}}{\mathop{v}}$), referred to as $\gamma$-votes, for all $y$ in $\mathcal{Y}$.
If the sum of all these $\delta$-votes is smaller than $C_k^\textsc{p}(\textit{x})$, there must be no more than $k$ labels (including $y_0$) with their $\gamma$-votes larger than or equal to those of $y_0$, i.e., $[\sum\nolimits_{y \in \mathcal{Y}}\sideset{^\delta}{_{y}}{\mathop{v}}<C_k^\textsc{p}(\textit{x})]\implies[|\{y\in\mathcal{Y}\mid  \sideset{^\gamma}{_{y}}{\mathop{v}}\geq \sideset{^\gamma}{_{y_0}}{\mathop{v}}\}|\leq k]$.


\end{thm}
\begin{proof} 
Let $\mathbb{Y}$ denote the set of labels to receive more votes, i.e., $\mathbb{Y}=\{y\in\mathcal{Y}\mid \sideset{^\delta}{_{y}}{\mathop{v}}>0\}$.
For ease of presentation, let $W_\alpha$ denote the set $\{y\in\mathcal{Y}\mid  \sideset{^\alpha}{^\textsc{p}_{y}}{\mathop{v}}(\textit{x})\geq \sideset{^\alpha}{^\textsc{p}_{y_0}}{\mathop{v}}(\textit{x})\}$ and
$W_\gamma$ denote the set $\{y\in\mathcal{Y}\mid  \sideset{^\gamma}{_{y}}{\mathop{v}}\geq \sideset{^\gamma}{_{y_0}}{\mathop{v}}\}$.
To maximize the interest of an attacker (to minimize the sum of all $\delta$-votes), the attacker should not add any votes to $y_0$ (i.e., $\sideset{^\delta}{_{y_0}}{\mathop{v}}= 0$).
We already know
$n=|W_\alpha|$.
Then, by the definitions of $W_\alpha$ and $W_\gamma$, we know $W_\alpha \subset W_\gamma$. 
Therefore, 
the set $\mathbb{Y}$ should contain at least $k-n+1$ labels not in the set $W_\alpha$ (i.e., $[\forall y\in \mathbb{Y}, \sideset{^\alpha}{^\textsc{p}_{y}}{\mathop{v}}(\textit{x})-\sideset{^\alpha}{^\textsc{p}_{y_0}}{\mathop{v}}(\textit{x})<0]\land|\mathbb{Y}|\geq k-n+1$) to make $|W_\gamma| > k$ hold. 
For each $y \in \mathbb{Y}$,  $\sideset{^\delta}{_{y}}{\mathop{v}}$ will be smallest to achieve $\sideset{^\gamma}{_{y}}{\mathop{v}}\geq \sideset{^\gamma}{_{y_0}}{\mathop{v}}$ if $y$ and $y_0$ share the same $\gamma$-vote (i.e., $\sideset{^\gamma}{_{y_0}}{\mathop{v}}=\sideset{^\gamma}{_{y}}{\mathop{v}}$), 
which happens when $\sideset{^\delta}{_{y}}{\mathop{v}} = \sideset{^\alpha}{^\textsc{p}_{y_0}}{\mathop{v}}(\textit{x})-\sideset{^\alpha}{^\textsc{p}_{y}}{\mathop{v}}(\textit{x})$.
Thus, $
\sum_{y\in\mathbb{Y}}[\sideset{^\alpha}{^\textsc{p}_{y_0}}{\mathop{v}}(\textit{x})-\sideset{^\alpha}{^\textsc{p}_{y}}{\mathop{v}}(\textit{x})]$ is the {smallest tie cost} to make these $|\mathbb{Y}|$ labels satisfy the condition $\sideset{^\gamma}{_{y}}{\mathop{v}}(\textit{x})\geq \sideset{^\gamma}{_{y_0}}{\mathop{v}}(\textit{x})$ for all $y$ in $\mathbb{Y}$.
So, the term
$
\min_{\mathbb{Y} \subseteq \mathcal{Y} \land 
[\forall y\in \mathbb{Y}, \sideset{^\alpha}{^\textsc{p}_{y}}{\mathop{v}}(\textit{x})-\sideset{^\alpha}{^\textsc{p}_{y_0}}{\mathop{v}}(\textit{x})<0] \land
|\mathbb{Y}|\geq k-n+1}
\sum_{y\in\mathbb{Y}}[\sideset{^\alpha}{^\textsc{p}_{y_0}}{\mathop{v}}(\textit{x})-\sideset{^\alpha}{^\textsc{p}_{y}}{\mathop{v}}(\textit{x})]$
returns the {smallest tie cost} that makes $|W_\gamma| > k$ hold.
As this cost would not be further reduced as $|\mathbb{Y}|$ increases, 
this term will yield the same cost as
$
\min_{\mathbb{Y} \subseteq \mathcal{Y} \land 
[\forall y\in \mathbb{Y}, \sideset{^\alpha}{^\textsc{p}_{y}}{\mathop{v}}(\textit{x})-\sideset{^\alpha}{^\textsc{p}_{y_0}}{\mathop{v}}(\textit{x})<0] \land
|\mathbb{Y}|=k-n+1}
\sum_{y\in\mathbb{Y}}[\sideset{^\alpha}{^\textsc{p}_{y_0}}{\mathop{v}}(\textit{x})-\sideset{^\alpha}{^\textsc{p}_{y}}{\mathop{v}}(\textit{x})]$. 
In Eq.~\ref{eq:cost}, the term $(k-n+1)\sideset{^\alpha}{^\textsc{p}_{y_0}}{\mathop{v}}(\textit{x})$ is a constant according to the given conditions in Thm.~\ref{thm:thm_new_tk_lowest_cost}. Eq.~\ref{eq:cost} can thus also be rewritten as the summation of the individual differences between the $\alpha$-votes of $y_0$ and $y$: 
$
C_k^\textsc{p}(\textit{x})=\min_{\mathbb{Y} \subseteq \mathcal{Y} \land 
\Phi(k, y_0, \textsc{p}, \mathbb{Y})}
\sum_{y\in\mathbb{Y}}[\sideset{^\alpha}{^\textsc{p}_{y_0}}{\mathop{v}}(\textit{x})-\sideset{^\alpha}{^\textsc{p}_{y}}{\mathop{v}}(\textit{x})]$. 
By the definition of $\Phi(k, y_0, \textsc{p}, \mathbb{Y})$, $C_k^\textsc{p}(\textit{x})$ is the same as the abovementioned term that returns the {smallest tie cost},
which means that the attacker must spend  at least 
$C_k^\textsc{p}(\textit{x})$ additional votes to 
make more than $k$ labels (including $y_0$) with $\gamma$-votes higher than or equal to those of $y_0$
(i.e., 
$[|\{y\in\mathcal{Y}\mid  \sideset{^\gamma}{^\textsc{p}_{y}}{\mathop{v}}(\textit{x})\geq \sideset{^\gamma}{^\textsc{p}_{y_0}}{\mathop{v}}(\textit{x})\}|> k]
\implies
[\sum_{y\in\mathbb{Y}}\sideset{^\delta}{^\textsc{p}_{y}}{\mathop{v}}\geq C_k^\textsc{p}(\textit{x})]$).
(Note, in logic, the two logic formulas 
$[A \implies B]$ and $[\neg B \implies \neg A]$ are equivalent.)
We obtain $[\sum\nolimits_{y \in \mathcal{Y}}\sideset{^\delta}{_{y}}{\mathop{v}}<C_k^\textsc{p}(\textit{x})]\implies[|\{y\in\mathcal{Y}\mid  \sideset{^\gamma}{^\textsc{p}_{y}}{\mathop{v}}(\textit{x})\geq \sideset{^\gamma}{^\textsc{p}_{y_0}}{\mathop{v}}(\textit{x})\}|\leq k]$.
%
%
%
\end{proof}







So, given a patch region and a specific malicious sample $\textit{x}'$ of the sample \textit{x}, we can apply  Thm.~\ref{thm:thm_new_tk_lowest_cost} to infer whether there are more than $k$ labels (including $y_0$) with their votes ($v_y(\textit{x}')$) larger than or equal to those of $y_0$, resulting in Cor. \ref{cor:special_case}.

\begin{cor}[$\beta$-votes of a malicious sample are an instance of $\delta$-votes in Thm.~\ref{thm:thm_new_tk_lowest_cost}]\label{cor:special_case}
Given a patch region $\textsc{p}$, a sample $\textit{x}$, and a malicious sample $\textit{x}'\in \{\textit{x}' \mid \textit{x}'=(\textsc{J}-\textsc{p})\odot \textit{x}+\textsc{p}\odot \textit{x}'' \}$. With all the conditions specified in Thm.~\ref{thm:thm_new_tk_lowest_cost} given, by Thm.~\ref{thm:thm_new_tk_lowest_cost}, 
$[\sum\nolimits_{y \in \mathcal{Y}}\sideset{^\beta}{_{y}^\textsc{p}}{\mathop{v}}(\textit{x}')<C_k^\textsc{p}(\textit{x})]\implies[|\{y\in\mathcal{Y}\mid  \sideset{}{_{y}}{\mathop{v}}(\textit{x}')\geq \sideset{}{_{y_0}}{\mathop{v}}(\textit{x}')\}|\leq k]$.
\end{cor}
\begin{proof}
Given $\textit{x}'$ with its \textsc{p}, 
by the definition of $\sideset{^\beta}{^\textsc{p}_{y}}{\mathop{v}}(\textit{x}')$, we know that $\sideset{^\beta}{^\textsc{p}_{y}}{\mathop{v}}(\textit{x}')$ is a non-negative integer for each $y\in\mathcal{Y}$, and $\sideset{}{_{y}}{\mathop{v}}(\textit{x}')=\sideset{^\alpha}{^\textsc{p}_{y}}{\mathop{v}}(\textit{x}')+\sideset{^\beta}{^\textsc{p}_{y}}{\mathop{v}}(\textit{x}')$, which can be rewritten as $\sideset{}{_{y}}{\mathop{v}}(\textit{x}')=\sideset{^\alpha}{^\textsc{p}_{y}}{\mathop{v}}(\textit{x})+\sideset{^\beta}{^\textsc{p}_{y}}{\mathop{v}}(\textit{x}')$  
since attackers cannot modify $\alpha$-votes.
In Thm.~\ref{thm:thm_new_tk_lowest_cost}, $\sideset{^\delta}{_{y}}{\mathop{v}}$ is an arbitrary non-negative integer subject to the condition specified in the theorem, and $\sideset{^\gamma}{_{y}}{\mathop{v}}$ = $\sideset{^\alpha}{^\textsc{p}_{y}}{\mathop{v}}(\textit{x})+\sideset{^\delta}{_{y}}{\mathop{v}}$. Therefore, after the symbol substitutions ($\sideset{^\gamma}{_{y}}{\mathop{v}}$ by $\sideset{}{_{y}}{\mathop{v}}(\textit{x}')$ and $\sideset{^\delta}{_{y}}{\mathop{v}}$
by $\sideset{^\beta}{^\textsc{p}_{y}}{\mathop{v}}(\textit{x}')$), our aim is to apply Thm.~\ref{thm:thm_new_tk_lowest_cost} to infer  
$[|\{y\in\mathcal{Y}\mid  \sideset{^\alpha}{^\textsc{p}_{y}}{\mathop{v}}(\textit{x})+\sideset{^\beta}{^\textsc{p}_{y}}{\mathop{v}}(\textit{x}')\geq \sideset{^\alpha}{^\textsc{p}_{y_0}}{\mathop{v}}(\textit{x})+\sideset{^\beta}{^\textsc{p}_{y_0}}{\mathop{v}}(\textit{x}')\}|
=
|\sideset{}{_{y}}{\mathop{v}}(\textit{x}')\geq \sideset{}{_{y_0}}{\mathop{v}}(\textit{x}')\}|
\leq
k
$.
Note that since we want to know the above-mentioned condition for
$\sideset{}{_{y}}{\mathop{v}}(\textit{x}')=\sideset{^\alpha}{^\textsc{p}_{y}}{\mathop{v}}(\textit{x})+\sideset{^\beta}{^\textsc{p}_{y}}{\mathop{v}}(\textit{x}')$, 
in the definition of $C_k^\textsc{p}(.)$ in Thm.~\ref{thm:thm_new_tk_lowest_cost}, 
the input is used for inputting to $\sideset{^\alpha}{^\textsc{p}_{y}}{\mathop{v}}(.)$, which is $\textit{x}$ (rather than $\textit{x}'$).
Thus, given both $\textit{x}'$ and $\textsc{p}$, 
if 
$[\sum\nolimits_{y \in \mathcal{Y}}\sideset{^\beta}{_{y}^\textsc{p}}{\mathop{v}}(\textit{x}')<C_k^\textsc{p}(\textit{x})]$ holds
, we can infer that
$
[|\{y\in\mathcal{Y}\mid  \sideset{}{_{y}}{\mathop{v}}(\textit{x}')\geq \sideset{}{_{y_0}}{\mathop{v}}(\textit{x}')\}|\leq k]$ holds.
(i.e., $[\sum\nolimits_{y \in \mathcal{Y}}\sideset{^\beta}{_{y}^\textsc{p}}{\mathop{v}}(\textit{x}')<C_k^\textsc{p}(\textit{x})]\implies[|\{y\in\mathcal{Y}\mid  \sideset{}{_{y}}{\mathop{v}}(\textit{x}')\geq \sideset{}{_{y_0}}{\mathop{v}}(\textit{x}')\}|\leq k]$).
\end{proof}
Checking all malicious samples of \textit{x} with the antecedent of the condition stated in Cor.~\ref{cor:special_case} is unlikely. 
 If we can imply this antecedent by a condition free from $x'$ and applicable to all malicious samples across all patch regions, we would still imply the consequence of that condition (see Thm. \ref{Thm:new_tk}).




\begin{thm}
[Certified Robustness for Top-$k$ Predictions with Analysis at Clean Level]\label{Thm:new_tk}
Suppose the maximum number of ablated regions that can overlap with any patch region $\textsc{p} \in \mathbb{P}$ (aka attack budget) is $\Delta$ (i.e., $\Delta=\max\{|\{\textsc{b}\in\mathbb{B}\mid \textsc{p} \in \mathbb{P} \land \textsc{p}\odot\textsc{b}\neq\textsc{O}\}|\}$). 
If the condition $[C_k^\textsc{p}(\textit{x})>\Delta]$ holds for all $\textsc{p}\in\mathbb{P}$ and the given conditions in Thm.~\ref{thm:thm_new_tk_lowest_cost} also hold,
then the condition $[\forall \textit{x}'\in\mathbb{A}(\textit{x}), y_0\in g^k(\textit{x}')]$ holds, i.e., \textbf{\textit{x} is \textit{k}-certified}.
\end{thm}
\begin{proof}
Suppose $\textsc{p}$ and \textit{x} are given. 
By Cor.~\ref{cor:special_case}, we know 
that for each given a malicious sample $\textit{x}'\in \{\textit{x}' \mid \textit{x}'=(\textsc{J}-\textsc{p})\odot \textit{x}+\textsc{p}\odot \textit{x}'' \}$, by Thm.~\ref{thm:thm_new_tk_lowest_cost}, 
$[\sum\nolimits_{y \in \mathcal{Y}}\sideset{^\beta}{_{y}^\textsc{p}}{\mathop{v}}(\textit{x}')<C_k^\textsc{p}(\textit{x})]\implies[|\{y\in\mathcal{Y}\mid  \sideset{}{_{y}}{\mathop{v}}(\textit{x}')\geq \sideset{}{_{y_0}}{\mathop{v}}(\textit{x}')\}|\leq k]$.
Given that $C_k^\textsc{p}(\textit{x})>\Delta$,
we have $C_k^\textsc{p}(\textit{x})>\Delta\geq |\{\textsc{b}\in\mathbb{B}\mid \textsc{p} \in \mathbb{P}
\land \textsc{p}\odot\textsc{b}\neq\textsc{O}\}|=\sum\nolimits_{y \in \mathcal{Y}}\sideset{^\beta}{_{y}^\textsc{p}}{\mathop{v}}(\textit{x}')$ 
for all $\textit{x}'\in \{\textit{x}' \mid \textit{x}'=(\textsc{J}-\textsc{p})\odot \textit{x}+\textsc{p}\odot \textit{x}'' \}$ by the definition of $\sideset{^\beta}{_{y}^\textsc{p}}{\mathop{v}}(\textit{x}')$. 
Therefore, for all $\textit{x}'\in \{\textit{x}' \mid \textit{x}'=(\textsc{J}-\textsc{p})\odot \textit{x}+\textsc{p}\odot \textit{x}'' \}$, 
we have $C_k^\textsc{p}(\textit{x})>\sum\nolimits_{y \in \mathcal{Y}}\sideset{^\beta}{_{y}^\textsc{p}}{\mathop{v}}(\textit{x}')$. 
Further, for all $\textit{x}'\in \{\textit{x}' \mid \textit{x}'=(\textsc{J}-\textsc{p})\odot \textit{x}+\textsc{p}\odot \textit{x}'' \}$, 
we obtain 
$[|\{y\in\mathcal{Y}\mid  \sideset{}{_{y}}{\mathop{v}}(\textit{x}')\geq \sideset{}{_{y_0}}{\mathop{v}}(\textit{x}')\}|\leq k]$ by Cor.~\ref{cor:special_case}.
Note that the condition $[C_k^\textsc{p}(\textit{x})>\Delta]$ holds for all $\textsc{p}\in\mathbb{P}$. By the definition of $\mathbb{A}(\textit{x})$, we know 
$[\forall \textit{x}'\in\mathbb{A}(\textit{x}), [|\{y\in\mathcal{Y}\mid  \sideset{}{_{y}}{\mathop{v}}(\textit{x}')\geq \sideset{}{_{y_0}}{\mathop{v}}(\textit{x}')\}|\leq k]$. 
Recalled that $g(.)$ ranks each label $y\in\mathcal{Y}$ by votes $v_y(.)$. Therefore, $\forall \textit{x}'\in\mathbb{A}(\textit{x})$, $y_0$ must contain in the top-$k$ labels (i.e., $\forall \textit{x}'\in\mathbb{A}(\textit{x}), y_0\in g^k(\textit{x}')$).
\end{proof}



\begin{Property}[%
Thm.~\ref{Thm:new_tk} provides a more precise analysis than 
Thm.~\ref{thm:old_tk}%
]
\label{pro: tighter}
A sample satisfying the antecedent stated in 
Thm.~\ref{thm:old_tk} (i.e., $[\forall \textsc{p}\in\mathbb{P},|\{y' \neq y_0\mid \overline{v}_{y'}^\textsc{p}(\textit{x})\geq \underline{v}_{y_0}^\textsc{p}(\textit{x})   \}| < k]$) will satisfy the antecedent stated in
Thm.~\ref{Thm:new_tk} 
(i.e., $[\forall \textsc{p}\in\mathbb{P}, C_k^\textsc{p}(x) >\Delta]$), but not vice versa.
\end{Property}


\begin{proof}
To prove the property, we need to prove $\text{\textcircled{1}} 
[\forall \textsc{p}\in\mathbb{P},|\{y' \neq y_0\mid \overline{v}_{y'}^\textsc{p}(\textit{x})\geq \underline{v}_{y_0}^\textsc{p}(\textit{x})   \}| < k] \Rightarrow [\forall \textsc{p}\in\mathbb{P}, C_k^\textsc{p}(\textit{x}) > \Delta]$.
$\text{\textcircled{2}} [\forall \textsc{p}\in\mathbb{P}, C_k^\textsc{p}(\textit{x}) > \Delta] \nRightarrow[\forall \textsc{p}\in\mathbb{P},|\{y' \neq y_0\mid \overline{v}_{y'}^\textsc{p}(\textit{x})\geq \underline{v}_{y_0}^\textsc{p}(\textit{x})   \}| < k]$.
We first prove $\text{\textcircled{1}}$.
Suppose a patch region \textsc{p} is given, and a given sample $x$ meets the antecedent in Thm.~\ref{thm:old_tk}.
Let $|\{y' \neq y_0\mid \overline{v}_{y'}^\textsc{p}(\textit{x})\geq \underline{v}_{y_0}^\textsc{p}(\textit{x})   \}|=m < k$.
By the definition of $\overline{v}_{y}(\textit{x}), \underline{v}_{y}(\textit{x}), \text{and }\sideset{^\alpha}{^\textsc{p}_{y}}{\mathop{v}}(\textit{x})$, 
we know $|\{y' \neq y_0\mid 
\sideset{^\alpha}{^\textsc{p}_{y'}}{\mathop{v}}(\textit{x}) + \Delta
\geq 
\sideset{^\alpha}{^\textsc{p}_{y_0}}{\mathop{v}}(\textit{x})   \}|=m <  k$,
which is 
$|\{y' \neq y_0\mid  
\sideset{^\alpha}{^\textsc{p}_{y_0}}{\mathop{v}}(\textit{x})   
-
\sideset{^\alpha}{^\textsc{p}_{y'}}{\mathop{v}}(\textit{x})
\leq
\Delta
\}
|=m < k$.
Recall $n=|\{y\in\mathcal{Y}\mid \sideset{^\alpha}{^\textsc{p}_{y}}{\mathop{v}}(\textit{x})\geq\sideset{^\alpha}{^\textsc{p}_{y_0}}{\mathop{v}}(\textit{x})\}|$.
\textcircled{1}(a) Suppose $n-1<k$.
We know 
$
C_k^\textsc{p}(\textit{x})=\min_{\mathbb{Y} \subseteq \mathcal{Y} \land 
\Phi(k, y_0, \textsc{p}, \mathbb{Y})}
\sum_{y\in\mathbb{Y}}[\sideset{^\alpha}{^\textsc{p}_{y_0}}{\mathop{v}}(\textit{x})-\sideset{^\alpha}{^\textsc{p}_{y}}{\mathop{v}}(\textit{x})]
,
$
where 
$
\Phi(k, y_0, \textsc{p}, \mathbb{Y}) := 
[|\mathbb{Y}|=k-n+1] \land [\forall y\in\mathbb{Y}, \sideset{^\alpha}{^\textsc{p}_{y}}{\mathop{v}}(\textit{x}) <\sideset{^\alpha}{^\textsc{p}_{y_0}}{\mathop{v}}(\textit{x}) ]
$. 
We can know that if the set $\mathbb{Y}$ contains any label $y''$ not in 
$\{y' \neq y_0\mid 
\sideset{^\alpha}{^\textsc{p}_{y_0}}{\mathop{v}}(\textit{x})   
-
\sideset{^\alpha}{^\textsc{p}_{y'}}{\mathop{v}}(\textit{x})
\leq
\Delta
\}
$ 
(which means $\sideset{^\alpha}{^\textsc{p}_{y_0}}{\mathop{v}}(\textit{x})   
-
\sideset{^\alpha}{^\textsc{p}_{y''}}{\mathop{v}}(\textit{x})
>
\Delta$)
,
then 
$C_k^\textsc{p}(x) > \Delta$ must hold.
To make it not happen, there should be $m-n+1$ labels remaining for $\mathbb{Y}$ (since each label except $y_0$ in $\{y\in\mathcal{Y}\mid \sideset{^\alpha}{^\textsc{p}_{y}}{\mathop{v}}(\textit{x})\geq\sideset{^\alpha}{^\textsc{p}_{y_0}}{\mathop{v}}(\textit{x})\}$ must be in $\{y' \neq y_0\mid 
\sideset{^\alpha}{^\textsc{p}_{y_0}}{\mathop{v}}(\textit{x})   
-
\sideset{^\alpha}{^\textsc{p}_{y'}}{\mathop{v}}(\textit{x})
\leq
\Delta
\}$. 
Labels in $\{y\in\mathcal{Y}\mid \sideset{^\alpha}{^\textsc{p}_{y}}{\mathop{v}}(\textit{x})\geq\sideset{^\alpha}{^\textsc{p}_{y_0}}{\mathop{v}}(\textit{x})\}$ cannot be in $\mathbb{Y}$ since they violate  $\Phi(k, y_0, \textsc{p}, \mathbb{Y})$). 
Also, we know $|\mathbb{Y}|=k-n+1$. Since $(k-n+1)-(m-n+1)=k-m>0$, we know $\mathbb{Y}$ cannot avoid including at least one label $y''$ that  
$\sideset{^\alpha}{^\textsc{p}_{y_0}}{\mathop{v}}(\textit{x})   
-
\sideset{^\alpha}{^\textsc{p}_{y''}}{\mathop{v}}(\textit{x})
>
\Delta$, making 
$C_k^\textsc{p}(\textit{x}) > \Delta$.
\textcircled{1}(b) Suppose $n-1\geq k$. Then $C_k^\textsc{p}(x)=0$ and we know 
$
|\{y\in\mathcal{Y}\mid \sideset{^\alpha}{^\textsc{p}_{y}}{\mathop{v}}(\textit{x})\geq\sideset{^\alpha}{^\textsc{p}_{y_0}}{\mathop{v}}(\textit{x})\}|\geq k+1
$.
Recall that by the definition of $\overline{v}_{y}(\textit{x})$, $\underline{v}_{y}(\textit{x})$, and $\sideset{^\alpha}{^\textsc{p}_{y}}{\mathop{v}}(\textit{x})$, 
we know $|\{y' \neq y_0\mid 
\sideset{^\alpha}{^\textsc{p}_{y'}}{\mathop{v}}(\textit{x}) + \Delta
\geq 
\sideset{^\alpha}{^\textsc{p}_{y_0}}{\mathop{v}}(\textit{x})   \}|=m < k$ from the given condition, contradicting to 
$
|\{y\in\mathcal{Y}\mid \sideset{^\alpha}{^\textsc{p}_{y}}{\mathop{v}}(\textit{x})\geq\sideset{^\alpha}{^\textsc{p}_{y_0}}{\mathop{v}}(\textit{x})\}|\geq k+1
$, which means this situation does not exist. 
Combining the two subcases, we know $C_k^\textsc{p}(\textit{x}) > \Delta$ holds for the given  $\textsc{p}$.
We also know the condition $|\{y' \neq y_0\mid \overline{v}_{y'}^\textsc{p}(\textit{x})\geq \underline{v}_{y_0}^\textsc{p}(\textit{x})   \}|=m < k$ holds for all $\textsc{p}\in\mathbb{P}$, therefore, $C_k^\textsc{p}(\textit{x}) > \Delta$ holds for all $\textsc{p}\in\mathbb{P}$.
To prove $\text{\textcircled{2}}$, we prove by counter-examples that meet
$[\forall \textsc{p}\in\mathbb{P}, C_k^\textsc{p} > \Delta]$
but not meets
$[\forall \textsc{p}\in\mathbb{P},|\{y' \neq y_0\mid \overline{v}_{y'}^\textsc{p}(\textit{x})\geq \underline{v}_{y_0}^\textsc{p}(\textit{x})   \}| < k]$.
Suppose the abovementioned sample $\textit{x}$ in \textcircled{1}(a) gets one more label $y_\delta$ that satisfies the condition $[\forall \textsc{p}\in\mathbb{P}, \overline{v}_{y_\delta}^\textsc{p}(\textit{x})\geq \underline{v}_{y_0}^\textsc{p}(\textit{x})\land 
\sideset{^\alpha}{^\textsc{p}_{y_\delta}}{\mathop{v}}(\textit{x})<\sideset{^\alpha}{^\textsc{p}_{y_0}}{\mathop{v}}(\textit{x})
]$, and others keep the same. 
Then, we will have $[\forall \textsc{p}\in\mathbb{P},|\{y' \neq y_0\mid \overline{v}_{y'}^\textsc{p}(\textit{x})\geq \underline{v}_{y_0}^\textsc{p}(\textit{x})   \}|=k]$,
violating the antecedent stated in Thm.~\ref{thm:old_tk}.
Let's denote $C_k^\textsc{p}(\textit{x})$ in \textcircled{1}(a) as $C^\textsc{p}$.
Then if the condition $[\forall \textsc{p}\in\mathbb{P}, \sideset{^\alpha}{^\textsc{p}_{y_\delta}}{\mathop{v}}(\textit{x}) < \sideset{^\alpha}{^\textsc{p}_{y_0}}{\mathop{v}}(\textit{x})+C^\textsc{p}-\Delta]$ holds, 
this sample can still satisfy the antecedent $[\forall \textsc{p}\in\mathbb{P}, C_k^\textsc{p}(\textit{x})
=
\min_{\mathbb{Y} \subseteq \mathcal{Y} \land 
\Phi(k, y_0, \textsc{p}, \mathbb{Y})}
\sum_{y\in\mathbb{Y}}[\sideset{^\alpha}{^\textsc{p}_{y_0}}{\mathop{v}}(\textit{x})-\sideset{^\alpha}{^\textsc{p}_{y}}{\mathop{v}}(\textit{x})]
\geq 
C^\textsc{p}+(\sideset{^\alpha}{^\textsc{p}_{y_0}}{\mathop{v}}(\textit{x}) -\sideset{^\alpha}{^\textsc{p}_{y_\delta}}{\mathop{v}}(\textit{x}) ) > \Delta]$ in Thm.~\ref{Thm:new_tk}.
\end{proof}
 


%
Intuitively, Property \ref{pro: tighter} shows that Thm.~\ref{Thm:new_tk} 
 \textbf{squeezes out \emph{unnecessary and inflated}} pairwise comparisons of upper/lower bounds for the top-$k$ labels in Thm.~\ref{thm:old_tk}. So, CostCert can certify more samples than a defender based on Strategy 2 for the same $k$ in general, providing a more precise certification analysis.
\subsection{Discussion}
CostCert can be further adapted to multi-class classification scenarios with multi-class labels, allowing $k$ guesses, where a sample may have more than one true label. 
For instance, if the safety requirement is that all true labels must be included in the top-$k$ prediction labels,
the user can apply the existing theory of CostCert (Thm.~\ref{Thm:new_tk}) by replacing $y_0$ in Eq.~\ref{eq:cost} with each of these true labels, one at a time; if all true labels can satisfy the antecedent in Thm.~\ref{Thm:new_tk}, CostCert will guarantee all these true labels within the top-$k$ prediction labels after patch attacks. 
We leave the formal formulation, proof, and evaluation of this extension as future work.

\section{Evaluation}\label{sec:eva}

\subsection{Research Questions}
We aim to answer the following research questions.

\begin{description} 
\item[RQ1] How does CostCert
 perform compared with baselines in terms of certified accuracy for top-$k$ predictions?
\item[RQ2] How does CostCert perform under strong adversaries in terms of patch size?
\end{description}

\subsection{Experimental Setup}
\subsubsection{Environment}

We implement and conduct the experiments on Python 3.8 and 
PyTorch 2.0.1. 
We train the base deep learning models and generate the mutants with their predicted labels on a GPU cluster equipped with V100 GPUs. We then conduct the data analysis of classifiers and certification analyzers 
on 
a Ubuntu 20.04 machine equipped with four 2080Ti GPUs.
\subsubsection{Datasets and Models}
We adopt three popular datasets, including the 1000-class ImageNet \cite{deng2009imagenet}, 100-class CIFAR100 \cite{krizhevsky2009learning}, 43-class GTSRB \cite{Stallkamp-IJCNN-2011} as our evaluation datasets, which are also adopted in previous works on patch robustness certification \cite{xiang2021patchguard,xiang2022patchcleanser, patchcensor,zhou2024crosscert,salman2022certified,zhang2020clipped}
. 
They contain 1.3M, 50k, and 39209 training images and 50k, 10k, and 12630 images for validation (ImageNet) and test (CIFAR100 and GTSRB), respectively. All samples are resized to $224\times 224$ pixels following \cite{salman2022certified}.
We adopt Vision Transformer (ViT) \cite{dosovitskiy2021an} as the base model of each defender, which achieves state-of-the-art voting-based recovery performances \cite{salman2022certified, chen2022towards} and high clean accuracy across many tasks.
We obtain the architecture and pre-trained weights (ViT-b16-224 with 86.6M parameters) from \emph{timm} \cite{vit-b16-224}. 
We adopt the mutant generation algorithm of column ablation from \cite{xiang2021patchguard}.
We adopt 19 as the size of column ablation in ViT, following the main experiments in \cite{salman2022certified}.

\subsubsection{Baselines}
\label{sec:baselines}

We compare CostCert (\textbf{CC}) with \textbf{CBN} \cite{zhang2020clipped}, PatchGuard (\textbf{PG}) \cite{xiang2021patchguard} and \textbf{PG$_\spadesuit$}.~\footnote{The works presented in \cite{jia2020certified} and \cite{jia2022almost} are unsuitable as our baselines. 
Unlike CC
, they do not offer deterministic certifications.
Ref \cite{jia2020certified} is for the L2-norm attacks.
Ref \cite{jia2022almost} can only defend against an attack up to 5 pixels.
}
We have reviewed PG's and CBN's certifications in Sections 
2.3.2 and 2.3.3.
PG$_\spadesuit$ is PG  with PG's masking strategy ablated,
 which shares the same certification theory with {CBN}.
We note that the code for the top-$k$ certification analyzers of PG and CBN are not open-sourced.
For fair comparisons,  we implement PG (PG-DS in \cite{xiang2021patchguard}) and PG$_\spadesuit$ with their top-$k$ certification analyzers on our infrastructure. All defenders share the same base models. Like the experiments in prior works \cite{zhou2023majority,zhou2024crosscert,li2022vip,salman2022certified,xiang2021patchguard,chen2022towards,levine2020randomized,patchcensor,zhang2020clipped}, a patch region is a square.
We follow \cite{levine2020randomized} to compute $\Delta$ and generate the patch region sets like \cite{xiang2021patchguard,zhou2023majority}.
We implement Alg.~\ref{alg:top-k-new} in CC with optimization, which returns the smallest $k$ sufficient to certify a sample as $k$-certified by checking Eq.~\ref{eq:cost} with $\mathbb{Y}$ that includes labels in descending order of $\alpha$-votes one at a time.
Like the evaluation reported in \cite{xiang2021patchguard,li2022vip,salman2022certified,zhou2024crosscert,patchcensor,metzen2021efficient,xiang2022patchcleanser,chen2022towards}, we extract the results from the original papers of PG \cite{xiang2021patchguard} and CBN \cite{zhang2020clipped} for comparisons whenever available, denoted as \textbf{PG$_\star$} and \textbf{CBN$_\star$}, respectively.


\subsubsection{Metrics}
\label{sec:metrics}
Let $\textit{x}$ be a sample with the true label $y_0$ in a test dataset $\mathbb{D}$ and $R = \langle g(x), c(x,k) \rangle$ be a voting-based recovery defender.

The sample $\textit{x}$ is called top-$k$ correct iff 
$y_0\in g^k(\textit{x})$. 
For a given $k$, \textbf{clean accuracy} is the fraction of $\mathbb{D}$ that are top-$k$ correct, defined as
$acc_{clean}(k)=\nicefrac{\mid\{\textit{x}\in\mathbb{D}\mid y_0\in g^k(\textit{x})\}\mid}{\mid\mathbb{D}\mid}$, and
%
\textbf{certified accuracy} is the ratio of $k$-certified samples to $|\mathbb{D}|$, defined as
$acc_{cert}(k)=\nicefrac{\mid\{\textit{x}\in\mathbb{D}\mid c(\textit{x},k)=\textit{True}\land y_0\in g^k(\textit{x})\}\mid}{\mid\mathbb{D}\mid}$.
Note that if  $\textit{x}$ is top-$k$ correct ($k$-certified), it is also top-$(k+1)$ correct ($(k+1)$-certified). 
The smallest $k$ sufficient to certify $x$ as $k$-certified is defined as 
\textbf{min\textit{k}}$(x)$ = $\min \{k | c(x, k)=\textit{True}\}$.

\subsubsection{Experimental Procedure}
%
For RQ1, following \cite{xiang2022patchcleanser,xiang2021patchguard,levine2020randomized,salman2022certified,zhou2023majority,chen2022towards},
we adopt the frequently-used patch sizes of 1\%  (measured as a patch region with side length of 23 pixels), 2\% (32 pixels), and 3\% (39 pixels) for ImageNet and 0.4\% (14 pixels) and 2.4\% (35 pixels) for the other two datasets. 
For RQ2, following the setting in RQ3 of \cite{patchcensor}, we vary the patch size by increasing it from 16 to 112 pixels with a step size of 16 pixels \cite{li2022vip,patchcensor}.
For each dataset, we run the fine-tuning script that follows the hyperparameters in \cite{salman2022certified} to obtain a fine-tuned base model, on which
we run the evaluation script to get the inference results (i.e., votes) of each sample $\textit{x}$.
We run the certification analyzer script on the inference results to compute $g(\textit{x})$ and min$k(\textit{x})$.

\begin{table}[t]
\caption{Results on ImageNet from Top-1 to Top-5}\label{tab:ImageNet}
\resizebox{\linewidth}{!}{
\begin{tabular}{|cc|ccc|ccc|}
\hline
\multicolumn{2}{|c|}{Accuracy (in \%)}                               & \multicolumn{3}{c|}{clean}                                                              & \multicolumn{3}{c|}{certified}                                                           \\ \hline
\multicolumn{2}{|c|}{Patch\,size}                              & \multicolumn{1}{c|}{1\%}             & \multicolumn{1}{c|}{2\%}             & 3\%             & \multicolumn{1}{c|}{1\%}             & \multicolumn{1}{c|}{2\%}             & 3\%             \\ \hline
\multicolumn{1}{|c|}{\multirow{5}{*}{{\rotatebox{90}{Top-1}}}} & CBN$_\star$ \cite{zhang2020clipped}           & \multicolumn{1}{c|}{62.0}          & \multicolumn{1}{c|}{62.0}          & 62.0          & \multicolumn{1}{c|}{n/a} & \multicolumn{1}{c|}{n/a} & n/a           \\ \cline{2-8} 
\multicolumn{1}{|c|}{}                       & PG$_\star$ \cite{xiang2021patchguard}            & \multicolumn{1}{c|}{55.1}          & \multicolumn{1}{c|}{54.6}          & 54.1          & \multicolumn{1}{c|}{32.2}          & \multicolumn{1}{c|}{26.0}          & 19.7          \\ \cline{2-8} 
\multicolumn{1}{|c|}{}                       & PG$_\spadesuit$ & \multicolumn{1}{c|}{\textbf{70.3}} & \multicolumn{1}{c|}{\textbf{70.3}} & \textbf{70.3} & \multicolumn{1}{c|}{46.1}          & \multicolumn{1}{c|}{41.0}          & 37.0          \\ \cline{2-8} 
\multicolumn{1}{|c|}{}                       & PG        & \multicolumn{1}{c|}{69.4}          & \multicolumn{1}{c|}{68.9}          & 68.5          & \multicolumn{1}{c|}{\textbf{48.1}} & \multicolumn{1}{c|}{\textbf{43.5}} & \textbf{40.0} \\ \cline{2-8} 
\multicolumn{1}{|c|}{}                       & CC            & \multicolumn{1}{c|}{\textbf{70.3}} & \multicolumn{1}{c|}{\textbf{70.3}} & \textbf{70.3} & \multicolumn{1}{c|}{46.1}          & \multicolumn{1}{c|}{41.0}          & 37.0          \\ \hline
\multicolumn{1}{|c|}{\multirow{4}{*}{{\rotatebox{90}{Top-2}}}} & PG$_\star$ \cite{xiang2021patchguard}            & \multicolumn{1}{c|}{65.9}          & \multicolumn{1}{c|}{65.5}          & 64.9          & \multicolumn{1}{c|}{48.3}          & \multicolumn{1}{c|}{43.8}          & 38.2          \\ \cline{2-8} 
\multicolumn{1}{|c|}{}                       & PG$_\spadesuit$ & \multicolumn{1}{c|}{\textbf{80.2}} & \multicolumn{1}{c|}{\textbf{80.2}} & \textbf{80.2} & \multicolumn{1}{c|}{53.1}          & \multicolumn{1}{c|}{47.2}          & 42.8          \\ \cline{2-8} 
\multicolumn{1}{|c|}{}                       & PG        & \multicolumn{1}{c|}{79.0}          & \multicolumn{1}{c|}{78.4}          & 77.8          & \multicolumn{1}{c|}{\textbf{57.2}} & \multicolumn{1}{c|}{52.5}          & 48.7          \\ \cline{2-8} 
\multicolumn{1}{|c|}{}                       & CC            & \multicolumn{1}{c|}{\textbf{80.2}} & \multicolumn{1}{c|}{\textbf{80.2}} & \textbf{80.2} & \multicolumn{1}{c|}{\textbf{57.2}} & \multicolumn{1}{c|}{\textbf{52.8}} & \textbf{49.5} \\ \hline
\multicolumn{1}{|c|}{\multirow{4}{*}{{\rotatebox{90}{Top-3}}}} & PG$_\star$ \cite{xiang2021patchguard}            & \multicolumn{1}{c|}{71.3}          & \multicolumn{1}{c|}{70.8}          & 70.2          & \multicolumn{1}{c|}{52.2}          & \multicolumn{1}{c|}{48.7}          & 44.1          \\ \cline{2-8} 
\multicolumn{1}{|c|}{}                       & PG$_\spadesuit$ & \multicolumn{1}{c|}{\textbf{84.2}} & \multicolumn{1}{c|}{\textbf{84.2}} & \textbf{84.2} & \multicolumn{1}{c|}{56.0}          & \multicolumn{1}{c|}{49.8}          & 45.1          \\ \cline{2-8} 
\multicolumn{1}{|c|}{}                       & PG        & \multicolumn{1}{c|}{82.7}          & \multicolumn{1}{c|}{82.0}          & 81.3          & \multicolumn{1}{c|}{61.6}          & \multicolumn{1}{c|}{57.1}          & 53.3          \\ \cline{2-8} 
\multicolumn{1}{|c|}{}                       & CC            & \multicolumn{1}{c|}{\textbf{84.2}} & \multicolumn{1}{c|}{\textbf{84.2}} & \textbf{84.2} & \multicolumn{1}{c|}{\textbf{63.2}} & \multicolumn{1}{c|}{\textbf{59.4}} & \textbf{56.5} \\ \hline
\multicolumn{1}{|c|}{\multirow{4}{*}{{\rotatebox{90}{Top-4}}}} & PG$_\star$ \cite{xiang2021patchguard}            & \multicolumn{1}{c|}{74.6}          & \multicolumn{1}{c|}{74.2}          & 73.7          & \multicolumn{1}{c|}{53.9}          & \multicolumn{1}{c|}{51.3}          & 47.4          \\ \cline{2-8} 
\multicolumn{1}{|c|}{}                       & PG$_\spadesuit$ & \multicolumn{1}{c|}{\textbf{86.3}} & \multicolumn{1}{c|}{\textbf{86.3}} & \textbf{86.3} & \multicolumn{1}{c|}{57.5}          & \multicolumn{1}{c|}{51.2}          & 46.4          \\ \cline{2-8} 
\multicolumn{1}{|c|}{}                       & PG        & \multicolumn{1}{c|}{84.7}          & \multicolumn{1}{c|}{84.0}          & 83.3          & \multicolumn{1}{c|}{64.5}          & \multicolumn{1}{c|}{60.0}          & 56.2          \\ \cline{2-8} 
\multicolumn{1}{|c|}{}                       & CC            & \multicolumn{1}{c|}{\textbf{86.3}} & \multicolumn{1}{c|}{\textbf{86.3}} & \textbf{86.3} & \multicolumn{1}{c|}{\textbf{66.9}} & \multicolumn{1}{c|}{\textbf{63.6}} & \textbf{60.8} \\ \hline
\multicolumn{1}{|c|}{\multirow{5}{*}{{\rotatebox{90}{Top-5}}}} & CBN$_\star$ \cite{zhang2020clipped}           & \multicolumn{1}{c|}{83.6}          & \multicolumn{1}{c|}{83.6}          & 83.6          & \multicolumn{1}{c|}{20.0}          & \multicolumn{1}{c|}{10.0}          & 5.0          \\ \cline{2-8} 
\multicolumn{1}{|c|}{}                       & PG$_\star$ \cite{xiang2021patchguard}            & \multicolumn{1}{c|}{77.0}          & \multicolumn{1}{c|}{76.6}          & 76.2          & \multicolumn{1}{c|}{54.8}          & \multicolumn{1}{c|}{52.9}          & 49.6          \\ \cline{2-8} 
\multicolumn{1}{|c|}{}                       & PG$_\spadesuit$ & \multicolumn{1}{c|}{\textbf{87.7}} & \multicolumn{1}{c|}{\textbf{87.7}} & \textbf{87.7} & \multicolumn{1}{c|}{58.5}          & \multicolumn{1}{c|}{52.1}          & 47.2          \\ \cline{2-8} 
\multicolumn{1}{|c|}{}                       & PG        & \multicolumn{1}{c|}{86.0}          & \multicolumn{1}{c|}{85.1}          & 84.4          & \multicolumn{1}{c|}{66.5}          & \multicolumn{1}{c|}{62.2}          & 58.3          \\ \cline{2-8} 
\multicolumn{1}{|c|}{}                       & CC            & \multicolumn{1}{c|}{\textbf{87.7}} & \multicolumn{1}{c|}{\textbf{87.7}} & \textbf{87.7} & \multicolumn{1}{c|}{\textbf{69.6}} & \multicolumn{1}{c|}{\textbf{66.4}} & \textbf{64.0} \\ \hline
\end{tabular}
}
\end{table}

\begin{table}[t]
{
\caption{Top-1 to Top-5 Results on GTSRB and CIFAR100}\label{tab:GTSRB&CIFAR100}
\setlength{\tabcolsep}{0.2em}
\resizebox{\linewidth}{!}{
\begin{tabular}{|cc|cccc|cccc|}
\hline
\multicolumn{2}{|c|}{Dataset}                                & \multicolumn{4}{c|}{GTSRB}                                                                                                   & \multicolumn{4}{c|}{CIFAR100}                                                                                                \\ \hline
\multicolumn{2}{|c|}{Patchsize}                              & \multicolumn{2}{c|}{0.4\%}                                                & \multicolumn{2}{c|}{2.4\%}                           & \multicolumn{2}{c|}{0.4\%}                                                & \multicolumn{2}{c|}{2.4\%}                           \\ \hline
\multicolumn{2}{|c|}{Accuracy (in \%)}                               & \multicolumn{1}{c|}{clean}         & \multicolumn{1}{c|}{certified}     & \multicolumn{1}{c|}{clean}         & certified     & \multicolumn{1}{c|}{clean}         & \multicolumn{1}{c|}{certified}     & \multicolumn{1}{c|}{clean}         & certified     \\ \hline
\multicolumn{1}{|c|}{\multirow{3}{*}{\rotatebox{90}{Top-1}}} & PG$_\spadesuit$ & \multicolumn{1}{c|}{\textbf{73.0}} & \multicolumn{1}{c|}{47.3}          & \multicolumn{1}{c|}{\textbf{73.0}} & 30.1          & \multicolumn{1}{c|}{\textbf{72.7}} & \multicolumn{1}{c|}{53.7}          & \multicolumn{1}{c|}{\textbf{72.7}} & 40.7          \\ \cline{2-10} 
\multicolumn{1}{|c|}{}                       & PG        & \multicolumn{1}{c|}{72.3}          & \multicolumn{1}{c|}{\textbf{47.7}} & \multicolumn{1}{c|}{68.9}          & \textbf{31.5} & \multicolumn{1}{c|}{72.4}          & \multicolumn{1}{c|}{\textbf{54.5}} & \multicolumn{1}{c|}{71.4}          & \textbf{43.3} \\ \cline{2-10} 
\multicolumn{1}{|c|}{}                       & CC            & \multicolumn{1}{c|}{\textbf{73.0}} & \multicolumn{1}{c|}{47.3}          & \multicolumn{1}{c|}{\textbf{73.0}} & 30.1          & \multicolumn{1}{c|}{\textbf{72.7}} & \multicolumn{1}{c|}{53.7}          & \multicolumn{1}{c|}{\textbf{72.7}} & 40.7          \\ \hline
\multicolumn{1}{|c|}{\multirow{3}{*}{\rotatebox{90}{Top-2}}} & PG$_\spadesuit$ & \multicolumn{1}{c|}{\textbf{85.8}} & \multicolumn{1}{c|}{59.4}          & \multicolumn{1}{c|}{\textbf{85.8}} & 38.6          & \multicolumn{1}{c|}{\textbf{82.3}} & \multicolumn{1}{c|}{61.2}          & \multicolumn{1}{c|}{\textbf{82.3}} & 47.5          \\ \cline{2-10} 
\multicolumn{1}{|c|}{}                       & PG        & \multicolumn{1}{c|}{84.3}          & \multicolumn{1}{c|}{60.3}          & \multicolumn{1}{c|}{79.4}          & 41.8          & \multicolumn{1}{c|}{82.1}          & \multicolumn{1}{c|}{63.5}          & \multicolumn{1}{c|}{81.0}          & 52.5          \\ \cline{2-10} 
\multicolumn{1}{|c|}{}                       & CC            & \multicolumn{1}{c|}{\textbf{85.8}} & \multicolumn{1}{c|}{\textbf{62.1}} & \multicolumn{1}{c|}{\textbf{85.8}} & \textbf{46.6} & \multicolumn{1}{c|}{\textbf{82.3}} & \multicolumn{1}{c|}{\textbf{63.6}} & \multicolumn{1}{c|}{\textbf{82.3}} & \textbf{53.6} \\ \hline
\multicolumn{1}{|c|}{\multirow{3}{*}{\rotatebox{90}{Top-3}}} & PG$_\spadesuit$ & \multicolumn{1}{c|}{\textbf{90.6}} & \multicolumn{1}{c|}{65.1}          & \multicolumn{1}{c|}{\textbf{90.6}} & 43.3          & \multicolumn{1}{c|}{\textbf{86.1}} & \multicolumn{1}{c|}{64.8}          & \multicolumn{1}{c|}{\textbf{86.1}} & 50.5          \\ \cline{2-10} 
\multicolumn{1}{|c|}{}                       & PG        & \multicolumn{1}{c|}{88.7}          & \multicolumn{1}{c|}{67.2}          & \multicolumn{1}{c|}{84.5}          & 48.8          & \multicolumn{1}{c|}{85.9}          & \multicolumn{1}{c|}{68.3}          & \multicolumn{1}{c|}{84.9}          & 57.9          \\ \cline{2-10} 
\multicolumn{1}{|c|}{}                       & CC            & \multicolumn{1}{c|}{\textbf{90.6}} & \multicolumn{1}{c|}{\textbf{69.9}} & \multicolumn{1}{c|}{\textbf{90.6}} & \textbf{56.3} & \multicolumn{1}{c|}{\textbf{86.1}} & \multicolumn{1}{c|}{\textbf{69.5}} & \multicolumn{1}{c|}{\textbf{86.1}} & \textbf{60.2} \\ \hline
\multicolumn{1}{|c|}{\multirow{3}{*}{\rotatebox{90}{Top-4}}} & PG$_\spadesuit$ & \multicolumn{1}{c|}{\textbf{93.0}} & \multicolumn{1}{c|}{68.3}          & \multicolumn{1}{c|}{\textbf{93.0}} & 46.1          & \multicolumn{1}{c|}{\textbf{88.7}} & \multicolumn{1}{c|}{66.7}          & \multicolumn{1}{c|}{\textbf{88.7}} & 52.2          \\ \cline{2-10} 
\multicolumn{1}{|c|}{}                       & PG        & \multicolumn{1}{c|}{90.9}          & \multicolumn{1}{c|}{71.7}          & \multicolumn{1}{c|}{87.2}          & 54.1          & \multicolumn{1}{c|}{88.0}          & \multicolumn{1}{c|}{71.6}          & \multicolumn{1}{c|}{86.9}          & 61.2          \\ \cline{2-10} 
\multicolumn{1}{|c|}{}                       & CC            & \multicolumn{1}{c|}{\textbf{93.0}} & \multicolumn{1}{c|}{\textbf{75.6}} & \multicolumn{1}{c|}{\textbf{93.0}} & \textbf{62.3} & \multicolumn{1}{c|}{\textbf{88.7}} & \multicolumn{1}{c|}{\textbf{73.3}} & \multicolumn{1}{c|}{\textbf{88.7}} & \textbf{64.8} \\ \hline
\multicolumn{1}{|c|}{\multirow{3}{*}{\rotatebox{90}{Top-5}}} & PG$_\spadesuit$ & \multicolumn{1}{c|}{\textbf{94.3}} & \multicolumn{1}{c|}{70.4}          & \multicolumn{1}{c|}{\textbf{94.3}} & 48.1          & \multicolumn{1}{c|}{\textbf{90.2}} & \multicolumn{1}{c|}{68.1}          & \multicolumn{1}{c|}{\textbf{90.2}} & 53.3          \\ \cline{2-10} 
\multicolumn{1}{|c|}{}                       & PG        & \multicolumn{1}{c|}{92.0}          & \multicolumn{1}{c|}{75.2}          & \multicolumn{1}{c|}{88.4}          & 58.4          & \multicolumn{1}{c|}{89.4}          & \multicolumn{1}{c|}{74.0}          & \multicolumn{1}{c|}{88.0}          & 63.9          \\ \cline{2-10} 
\multicolumn{1}{|c|}{}                       & CC            & \multicolumn{1}{c|}{\textbf{94.3}} & \multicolumn{1}{c|}{\textbf{79.4}} & \multicolumn{1}{c|}{\textbf{94.3}} & \textbf{66.8} & \multicolumn{1}{c|}{\textbf{90.2}} & \multicolumn{1}{c|}{\textbf{75.9}} & \multicolumn{1}{c|}{\textbf{90.2}} & \textbf{68.4} \\ \hline
\end{tabular}
}
}
\end{table}

\subsection{Experimental Results and Data Analysis}
\subsubsection{Answering RQ1}
Tables \ref{tab:ImageNet} and \ref{tab:GTSRB&CIFAR100} summarize the results on 
$acc_{clean}(k)$ and $acc_{cert}(k)$ for $k=1$ to $5$ (shown as Top-1 to Top-5).
We note that the work of Zhang et al. \cite{zhang2020clipped} only provides results for CBN in $acc_{clean}(k=1)$, $acc_{clean}(k=5)$, and $acc_{cert}(k=5)$ on ImageNet.
We highlight the best results in each combination of patch size, $k$'s value, dataset, and metric. 


We first discuss the result on ImageNet.
In Table~\ref{tab:ImageNet}, 
CC achieves the best results for $k$ = 2 to 5 on both metrics across all patch sizes 
compared with CBN$_\star$, PG$_\star$, PG$_\spadesuit$, and PG. For instance, in certified accuracy in top-5 (i.e.,  $acc_{cert}(k=5)$), CC achieves 69.6\%, 66.4\%, and 64.0\% if the patch sizes are 1\%, 2\%, and 3\%, which are higher than 
these of PG and PG$_\spadesuit$ by 3.1\%, 4.2\%, and 5.7\% 
and
by 11.1\%, 14.3\%, and 16.8\%, 
respectively: 
The gap between the certified accuracy of CC and those of PG or PG$_\spadesuit$ increases when the patch size increases.%
\footnote{ 
The reason is that the methodology of PG and PG$_\spadesuit$ to certify samples is to compare the bounds between labels individually, leading to an overestimation of the attack budgets in the certification analyses 
(see Sections 2.4 and 3.1).
When the patch size increases, the overestimation from this methodology also increases, while the methodology of CC is more precise (see Property~\ref{pro: tighter}).}

Although PG achieves the best certified accuracy in top-1 (i.e., $acc_{cert}(k=1)$), it has a lower clean accuracy than CC and PG$_\spadesuit$. When it comes to top-2 (i.e., $acc_{cert}(k=2)$), PG only achieves the same certified accuracy as CC if the patch size is 1\% {but lower} for the two larger patch sizes, while its clean accuracy suffers 1.2\%, 1.8\%, and 2.4\% drops compared with CC and PG$_\spadesuit$. 
The gap in certified accuracy between CC and PG also increases as $k$ increases --- PG's masking strategy (that sacrifices the clean accuracy) cannot narrow this gap.

We also observe that CC and PG$_\spadesuit$ share the same clean accuracy, PG has higher clean and certified accuracy than PG$_\star$, 
and PG has higher certified accuracy but lower clean accuracy than PG$_\spadesuit$. 
These observations reconcile our experiments setting that CC and PG$_\spadesuit$ use the same classifier $g(.)$ and our assumption that PG equipped with ViT is more powerful than PG in its original paper \cite{xiang2021patchguard}, and PG's masking strategy sacrifices clean accuracy.

\begin{figure*}[tbp] 
\centering
\begin{subfigure}[t]{0.325\textwidth}
\centering
\includegraphics[width=\textwidth]{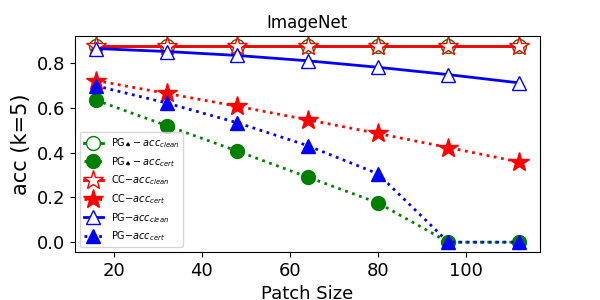}
\end{subfigure}
\begin{subfigure}[t]{0.325\textwidth}
\centering
\includegraphics[width=\textwidth]{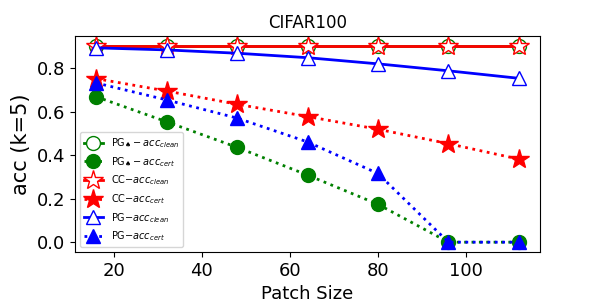}
\end{subfigure}
\begin{subfigure}[t]{0.325\textwidth}
\centering
\includegraphics[width=\textwidth]{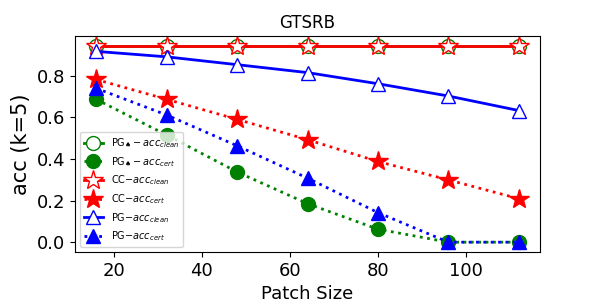}
\end{subfigure}
\begin{subfigure}[t]{0.325\textwidth}
\centering
\includegraphics[width=\textwidth]{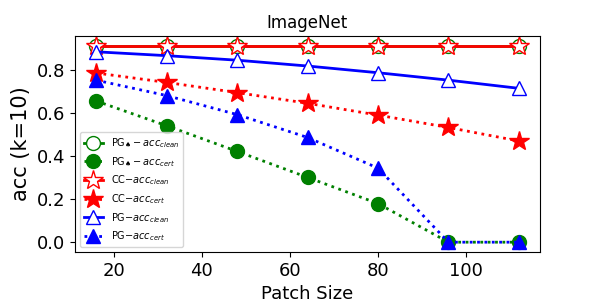}
\end{subfigure}
\begin{subfigure}[t]{0.325\textwidth}
\centering
\includegraphics[width=\textwidth]{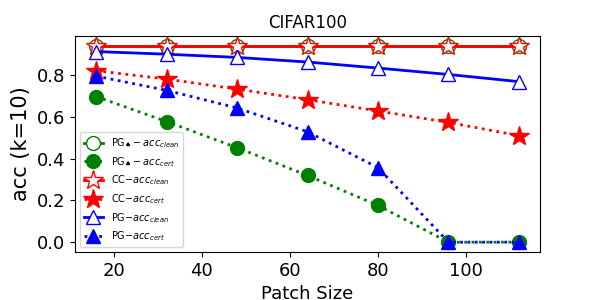}
\end{subfigure}
\begin{subfigure}[t]{0.325\textwidth}
\centering
\includegraphics[width=\textwidth]{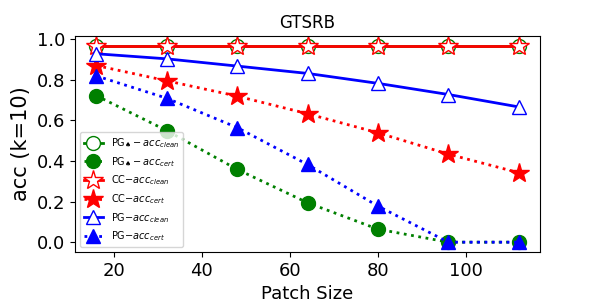}
\end{subfigure}
\captionsetup{font={small}}
\caption{Certified accuracy and clean accuracy of PG, PG$_\spadesuit$, and CC. The three columns of plots from left to right are for ImageNet, CIFAR100, and GTSRB, respectively. In a column, the $y$-axis in the upper plot is $acc_{clean}(k=5)$ and $acc_{cert}(k=5)$ and the $x$-axis is the patch size in pixels, and the lower plot shows the results for $acc_{clean}(k=10)$ and $acc_{cert}(k=10)$ with the same interpretations for both axes as the upper plot, depicted with solid lines for $acc_{clean}(k)$ and dotted lines for $acc_{cert}(k)$ and markers (stars for CC, circles for PG$_\spadesuit$, and triangles for PG). As the patch size increases, the advantage in certified accuracy for CC over PG/PG$_\spadesuit$ increases, PG's clean accuracy drops, CC's certified accuracy drops almost linearly, and CC still certifies {29.9\%--57.3\%} samples even if the patch size reaches 96.
}
\label{fig:patch_size_vary}
\end{figure*}

Similarly, 
in Table \ref{tab:GTSRB&CIFAR100}, 
CC achieves the best results in all combinations of datasets and metrics with $k = 2$ to $5$.
Across all $k$, on average, CC achieves 4.0\% and 1.3\% higher certified accuracy and 3.7\% and 1.0\% higher clean accuracy than PG on GTSRB and CIFAR100, respectively.
All the observations in Table~\ref{tab:ImageNet} above-mentioned can also be observed in Table~\ref{tab:GTSRB&CIFAR100}.

Overall, CC gradually surpasses the peer techniques in $acc_{cert}(k)$ for $k$ varying from 2 to 5, showing the superior scalability in terms of $k$. Although PG's masking strategy sacrifices the clean accuracy compared to PG$_\spadesuit$ to make PG slightly outperform CC in certified accuracy when $k=1$, it reveals its severe limitation when $k\geq 2$.












\subsubsection{Answering RQ2}\label{sec:RQ2}
The six plots in Fig.~\ref{fig:patch_size_vary} 
summarize the results of 
$acc_{cert}(k=5)$ and $acc_{cert}(k=10)$ of CC, PG, and PG$_\spadesuit$ along with their clean accuracy from smaller patch sizes to larger patch sizes on all three datasets. 
We can see that CC's certified accuracy and clean accuracy steadily surpass these of its peers.
The drop in the clean accuracy of PG becomes more severe when the patch size gets larger. This is because PG's masking strategy makes the classifier lose more and more information about a sample for prediction when using a larger patch region (mask) 
{in the operation phase.}
%
Like what we observe in Tables \ref{tab:ImageNet} and \ref{tab:GTSRB&CIFAR100},
the trends of
the certified accuracy of PG and PG$_\spadesuit$ drop much more quickly than that of CC, which are vividly shown in Fig.~\ref{fig:patch_size_vary}. 
Fig.~\ref{fig:example} further illustrates the impact of increasing the patch size on 
 min$k$ for all three techniques.

Reconciling our results in Section 
2.4, in Fig.~\ref{fig:patch_size_vary}, the certified accuracy of PG and PG$_\spadesuit$ both drop to zero when the patch size is 96 (16\% in total area; the largest patch size in \cite{patchcensor,li2022vip} is 112, and that in \cite{xiang2024patchcure} is even 192 with the same $224\times 224$ size of image). We find that 2365 ImageNet samples are well-classified, i.e., all 224 mutants of these samples are correctly predicted to their corresponding true labels. 
By the equation in \cite{levine2020randomized}, we know $\Delta=114$ for this patch size. Intuitively, if adversarial patches (modeled by patch regions) cannot affect all 224
mutants in a sample, very likely, the remaining benign mutants would make the true label rank in a position much earlier than the trivial position (1000). 
Surprisingly, PG and PG$_\spadesuit$ can only trivially certify {all of them} with min$k=1000$, which practically fail to certify them and {are consistent with the results shown} in Fig.~\ref{fig:patch_size_vary}, whereas CC certifies {all of them} as $k$-certified samples with min$k = 2$.

\begin{figure}[t] 
\centering
\begin{subfigure}[t]{0.24\textwidth}
\centering
\includegraphics[width=\textwidth,height=0.7\textwidth]{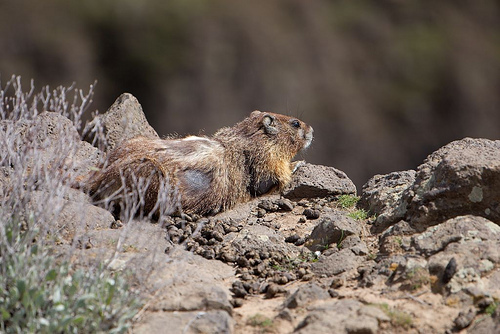}
\end{subfigure}
\begin{subfigure}[t]{0.24\textwidth}
\centering
\includegraphics[width=\textwidth,height=0.7\textwidth]{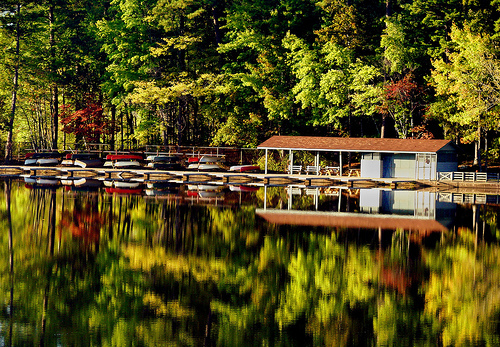}
\end{subfigure}
\captionsetup{font={small}}
\caption{Two examples from ImageNet \cite{deng2009imagenet}. On the left is an image of a marmot. It is top-1 correct for all PG$_\spadesuit$/PG/CC.
When the patch size is 16, PG$_\spadesuit$, PG, and CC each certify it as a 1-certified sample. 
When the patch size increases to 48, PG$_\spadesuit$, PG, and CC certify it as $k$-certified with min$k = 1000$ (trivial), min$k = 7$, and min$k=4$, respectively.
When the patch size increases to 80, CC certifies it as $k$-certified with min$k = 11$; nonetheless, both PG$_\spadesuit$ and PG certify it as $k$-certified with min$k = 1000$ (trivial). 
We can observe that PG and PG$_\spadesuit$ suffer from a quick performance degradation in terms of min$k$.
On the right is an image of a boathouse.  CC/PG/PG$_\spadesuit$ classify all its mutants to a boathouse as the top-1 label. However, when the patch size reaches 96 pixels, PG and PG$_\spadesuit$ can only trivially certify it with min$k$ = 1000, whereas the min$k$ for CC is only {2}.}
\label{fig:example}
\end{figure}

Overall, PG and PG$_\spadesuit$ are hard to tackle with stronger adversaries in terms of patch size due to the limitations in their designs. 
From our data analysis, they hit the wall {(certified accuracy drops to zero)} when a patch region (in area) is close to $\nicefrac{1}{5}$ of the image size. CostCert can bypass this limit from the results, showing the superior scalability in terms of patch sizes.
It also exhibits a more gentle drop than these two peers against larger patch sizes. 
These results are interesting.

\subsection{Threats to Validity}
We  
follow  \cite{xiang2021patchguard,li2022vip,salman2022certified,zhou2024crosscert,patchcensor,metzen2021efficient,xiang2022patchcleanser,chen2022towards} to compare CC with the published results of more defenders. 
We didn't compare CostCert with masking-based recovery defenders
\cite{xiang2022patchcleanser, xiang2024patchcure}.
Their certification is restricted to top-1 prediction due to their theoretical limitation; 
more importantly, they usually impractically need the patch shape and size when generating the mutants \cite{li2022vip}, 
which we do not.
Like the evaluations of above-mentioned existing defenders against patch attacks, 
we have not evaluated CostCert on real adversarial attacks as such evaluation alone cannot determine the certification of a sample, since it is not possible to guarantee the accurate detection of all potential malicious samples by empirical attacks.
The implementations of defenders may contain bugs. We have thoroughly tested them and fixed the bugs we found.  
CC takes 1048 seconds to analyze the whole ImageNet validation set, which is practical.
Like \cite{xiang2021patchguard, patchcensor,levine2020randomized,salman2022certified,zhou2024crosscert,zhou2023majority,chen2022towards}, we conduct experiments in single-patch situations in the shape of a square.
The experiment will be further generalized with more combinations of defenders, datasets, metrics, patch sizes and shapes, and base models.




\section{Related Work}
Gu et al. \cite{brown2017adversarial} propose patch attacks 
as a threat model for physically adversarial attacks in the real world \cite{levine2020randomized}. 
%
Significant research advances have been observed to counter patch attacks. 
Empirical defenders \cite{naseer2019local,hayes2018visible} can be undermined if an attacker knows their defense strategies \cite{chiang2020certified}. 
Therefore, robustness certification against patch attacks is emerging \cite{chiang2020certified,levine2020randomized,salman2022certified,chen2022towards,lin2021certified,xiang2021patchguard,metzen2021efficient,zhang2020clipped,xiang2022patchcleanser,xiang2024patchcure,zhou2023majority,patchcensor,li2022vip, zhou2024crosscert} to defense in such scenarios.
There are two research lines in this field: certified recovery \cite{chiang2020certified,levine2020randomized,salman2022certified,chen2022towards,lin2021certified,xiang2021patchguard,metzen2021efficient,zhang2020clipped,xiang2022patchcleanser,xiang2024patchcure,zhou2023majority,li2022vip} and certified detection \cite{patchcensor,li2022vip, zhou2024crosscert}. 
Certified recovery aims to ensure the recovery of the label of certified samples against the adversarial patch.
Certified detection relaxes the guarantee to provide warnings against such a patch \cite{xiang2021patchguard,xiangshort},
which 
suffer from frequent calls upon their fall-back strategies.

Chiang et al. \cite{chiang2020certified} propose the first work on certified recovery against patch attacks by interval-bound propagation, which is computationally expensive. Subsequently, DRS \cite{levine2020randomized} is proposed as the first voting-based recovery by using the margin of votes to provide the guarantee. After that, ViT \cite{salman2022certified}, ViP \cite{li2022vip}, and ECViT \cite{chen2022towards} demonstrate that Vision Transformer can achieve better certification results than convolutional neural networks.
CrossCert \cite{zhou2024crosscert} builds a detection defender atop a pair of recovery defenders.
CBN \cite{zhang2020clipped}, PatchGuard \cite{xiang2021patchguard}, MajorCert \cite{zhou2023majority}, and BagCert \cite{metzen2021efficient} investigate the relationships and constraints among the ablations and patch regions, aiming to eliminate infeasible cases, thereby improving the certified accuracy theoretically. 
On top of the commonly used voting mechanism, PatchGuard \cite{xiang2021patchguard} additionally proposes a heuristic masking strategy to enhance the certification.
However, to our knowledge, all but PatchGuard \cite{xiang2021patchguard} and CBN \cite{zhang2020clipped} reviewed above only focus on providing certification for the top-1 prediction. The newly proposed masking-based recovery \cite{xiang2022patchcleanser,xiang2024patchcure} outperformed the voting-based one in top-1 certification. 
However, they are hard to extend to the top-$k$ prediction/certification and usually require the patch size to generate corresponding mutants due to their theoretical limitation. 
Also, due to the practical need for top-$k$ predictions 
and attacks specifically designed for them 
\cite{zhang2022investigating,tursynbek2022geometry}, the need of the certification for top-$k$ predictions is emerging \cite{jia2020certified,jia2022almost}.
Yet, existing works \cite{xiang2021patchguard, zhang2020clipped} on patch robustness certification are not precise enough. 
CostCert advances the domain of certification theory by changing redundant comparisons of bounds into checking the overall cost. 

\section{Conclusion}
We have proposed a novel, scalable, precise, and cost-oriented certification technique CostCert.
CostCert is the \emph{first} certified recovery defender 
for top-$k$ predictions without inflation problems against patch attacks. 
Experiments have shown its high effectiveness and scalability compared to the baselines.
Our replication package can be found in \cite{CostCert}.


\bibliographystyle{IEEEtran}
\bibliography{references_short}

\begin{thebibliography}{10}
\providecommand{\url}[1]{#1}
\csname url@samestyle\endcsname
\providecommand{\newblock}{\relax}
\providecommand{\bibinfo}[2]{#2}
\providecommand{\BIBentrySTDinterwordspacing}{\spaceskip=0pt\relax}
\providecommand{\BIBentryALTinterwordstretchfactor}{4}
\providecommand{\BIBentryALTinterwordspacing}{\spaceskip=\fontdimen2\font plus
\BIBentryALTinterwordstretchfactor\fontdimen3\font minus \fontdimen4\font\relax}
\providecommand{\BIBforeignlanguage}[2]{{%
\expandafter\ifx\csname l@#1\endcsname\relax
\typeout{** WARNING: IEEEtran.bst: No hyphenation pattern has been}%
\typeout{** loaded for the language `#1'. Using the pattern for}%
\typeout{** the default language instead.}%
\else
\language=\csname l@#1\endcsname
\fi
#2}}
\providecommand{\BIBdecl}{\relax}
\BIBdecl

\bibitem{brown2017adversarial}
T.~B. Brown, D.~Man{\'e}, A.~Roy, M.~Abadi, and J.~Gilmer, ``Adversarial patch,'' \emph{ArXiv:1712.09665}, 2017.

\bibitem{szegedy2013intriguing}
C.~Szegedy, W.~Zaremba, I.~Sutskever, J.~Bruna, D.~Erhan, I.~Goodfellow, and R.~Fergus, ``Intriguing properties of neural networks,'' \emph{arXiv:1312.6199}, 2013.

\bibitem{liu2020bias}
A.~Liu, J.~Wang, X.~Liu, B.~Cao, C.~Zhang, and H.~Yu, ``Bias-based universal adversarial patch attack for automatic check-out,'' in \emph{Proceedings of ECCV}, 2020, pp. 395--410.

\bibitem{levine2020randomized}
A.~Levine and S.~Feizi, ``(de) randomized smoothing for certifiable defense against patch attacks,'' in \emph{Proceedings of NeurIPS}, 2020, pp. 6465--6475.

\bibitem{zhou2024crosscert}
Q.~Zhou, Z.~Wei, H.~Wang, B.~Jiang, and W.~Chan, ``Crosscert: A cross-checking detection approach to patch robustness certification for deep learning models,'' in \emph{PACMSE, Volume 1, Number FSE, Article 120}, 2024.

\bibitem{xiang2021patchguard}
C.~Xiang, A.~N. Bhagoji, V.~Sehwag, and P.~Mittal, ``Patchguard: A provably robust defense against adversarial patches via small receptive fields and masking,'' in \emph{Proceedings of USENIX Security}, 2021, pp. 2237--2254.

\bibitem{patchcensor}
Y.~Huang, L.~Ma, and Y.~Li, ``Patchcensor: Patch robustness certification for transformers via exhaustive testing,'' \emph{ACM TOSEM}, apr 2023, just Accepted.

\bibitem{zhang2020clipped}
Z.~Zhang, B.~Yuan, M.~McCoyd, and D.~Wagner, ``Clipped bagnet: Defending against sticker attacks with clipped bag-of-features,'' in \emph{2020 Proceedings of IEEE S\&P Workshops}, 2020, pp. 55--61.

\bibitem{xiang2022patchcleanser}
C.~Xiang, S.~Mahloujifar, and P.~Mittal, ``Patchcleanser: Certifiably robust defense against adversarial patches for any image classifier,'' in \emph{Proceedings of USENIX Security}, 2022, pp. 2065--2082.

\bibitem{salman2022certified}
H.~Salman, S.~Jain, E.~Wong, and A.~Madry, ``Certified patch robustness via smoothed vision transformers,'' in \emph{Proceedings of CVPR 2022}, pp. 15\,137--15\,147.

\bibitem{chiang2020certified}
P.~yeh Chiang, R.~Ni, A.~Abdelkader, C.~Zhu, C.~Studor, and T.~Goldstein, ``Certified defenses for adversarial patches,'' in \emph{ICLR}, 2020.

\bibitem{metzen2021efficient}
J.~H. Metzen and M.~Yatsura, ``Efficient certified defenses against patch attacks on image classifiers,'' in \emph{ICRL 2021}.

\bibitem{chen2022towards}
Z.~Chen, B.~Li, J.~Xu, S.~Wu, S.~Ding, and W.~Zhang, ``Towards practical certifiable patch defense with vision transformer,'' in \emph{Proceedings of CVPR 2022}, 2022, pp. 15\,148--15\,158.

\bibitem{li2022vip}
J.~Li, H.~Zhang, and C.~Xie, ``Vip: Unified certified detection and recovery for patch attack with vision transformers,'' in \emph{Proceedings of ECCV}, 2022, pp. 573--587.

\bibitem{zhou2023majority}
Q.~Zhou, Z.~Wei, H.~Wang, and W.~Chan, ``A majority invariant approach to patch robustness certification for deep learning models,'' in \emph{Proceedings of ASE 2023}, 2023, pp. 1790--1794.

\bibitem{hayes2018visible}
J.~Hayes, ``On visible adversarial perturbations \& digital watermarking,'' in \emph{2018 Proceedings of the IEEE CVPR Workshops}, 2018, pp. 1597--1604.

\bibitem{naseer2019local}
M.~Naseer, S.~Khan, and F.~Porikli, ``Local gradients smoothing: Defense against localized adversarial attacks,'' in \emph{Proceedings of WACV}, 2019, pp. 1300--1307.

\bibitem{lin2021certified}
W.-Y. Lin, F.~Sheikholeslami, L.~Rice, J.~Z. Kolter \emph{et~al.}, ``Certified robustness against adversarial patch attacks via randomized cropping,'' in \emph{ICML 2021 Workshop on Adversarial Machine Learning}, 2021.

\bibitem{xiang2024patchcure}
C.~Xiang, T.~Wu, S.~Dai, J.~Petit, S.~Jana, and P.~Mittal, ``Patchcure: Improving certifiable robustness, model utility, and computation efficiency of adversarial patch defenses,'' in \emph{33rd {USENIX} Security Symposium ({USENIX} Security)}, 2024.

\bibitem{berrada2018smooth}
L.~Berrada, A.~Zisserman, and P.~Mudigonda, ``Smooth loss functions for deep top-k classification,'' in \emph{ICLR}, 2018.

\bibitem{hsu2021fingat}
Y.-L. Hsu, Y.-C. Tsai, and C.-T. Li, ``Fingat: Financial graph attention networks for recommending top-$ k $ k profitable stocks,'' \emph{IEEE transactions on knowledge and data engineering}, vol.~35, no.~1, pp. 469--481, 2021.

\bibitem{petersen2022differentiable}
F.~Petersen, H.~Kuehne, C.~Borgelt, and O.~Deussen, ``Differentiable top-k classification learning,'' in \emph{ICML}.\hskip 1em plus 0.5em minus 0.4em\relax PMLR, 2022, pp. 17\,656--17\,668.

\bibitem{zhang2022investigating}
C.~Zhang, P.~Benz, A.~Karjauv, J.~W. Cho, K.~Zhang, and I.~S. Kweon, ``Investigating top-k white-box and transferable black-box attack,'' in \emph{Proceedings of CVPR 2022}, 2022, pp. 15\,085--15\,094.

\bibitem{tursynbek2022geometry}
N.~Tursynbek, A.~Petiushko, and I.~Oseledets, ``Geometry-inspired top-k adversarial perturbations,'' in \emph{Proceedings of the IEEE/CVF Winter Conference on Applications of Computer Vision}, 2022, pp. 3398--3407.

\bibitem{jia2020certified}
J.~Jia, X.~Cao, B.~Wang, and N.~Z. Gong, ``Certified robustness for top-k predictions against adversarial perturbations via randomized smoothing,'' in \emph{ICLR}, 2020.

\bibitem{jia2022almost}
J.~Jia, B.~Wang, X.~Cao, H.~Liu, and N.~Z. Gong, ``Almost tight l0-norm certified robustness of top-k predictions against adversarial perturbations,'' in \emph{ICLR}, 2022.

\bibitem{saha2023revisiting}
A.~Saha, S.~Yu, M.~S. Norouzzadeh, W.-Y. Lin, and C.~K. Mummadi, ``Revisiting image classifier training for improved certified robust defense against adversarial patches,'' \emph{TMLR}, 2023.

\bibitem{mccoyd2020minority}
M.~McCoyd, W.~Park, S.~Chen, N.~Shah, R.~Roggenkemper, M.~Hwang, J.~X. Liu, and D.~Wagner, ``Minority reports defense: Defending against adversarial patches,'' in \emph{Proceedings of ACNS}, 2020, pp. 564--582.

\bibitem{horn2012matrix}
R.~A. Horn and C.~R. Johnson, \emph{Matrix analysis}.\hskip 1em plus 0.5em minus 0.4em\relax Cambridge university press, 2012.

\bibitem{deng2009imagenet}
J.~Deng, W.~Dong, R.~Socher, L.-J. Li, K.~Li, and L.~Fei-Fei, ``Imagenet: A large-scale hierarchical image database,'' in \emph{Proceedings of CVPR}, 2009, pp. 248--255.

\bibitem{dosovitskiy2021an}
A.~Dosovitskiy, L.~Beyer, A.~Kolesnikov, D.~Weissenborn, X.~Zhai, T.~Unterthiner, M.~Dehghani, M.~Minderer, G.~Heigold, S.~Gelly, J.~Uszkoreit, and N.~Houlsby, ``An image is worth 16x16 words: Transformers for image recognition at scale,'' in \emph{ICLR 2021}, 2021.

\bibitem{flajolet2009analytic}
P.~Flajolet and R.~Sedgewick, \emph{Analytic combinatorics}.\hskip 1em plus 0.5em minus 0.4em\relax cambridge University press, 2009.

\bibitem{krizhevsky2009learning}
A.~Krizhevsky, G.~Hinton \emph{et~al.}, ``Learning multiple layers of features from tiny images,'' 2009.

\bibitem{Stallkamp-IJCNN-2011}
J.~Stallkamp, M.~Schlipsing, J.~Salmen, and C.~Igel, ``The {G}erman {T}raffic {S}ign {R}ecognition {B}enchmark: A multi-class classification competition,'' in \emph{IJCNN}, 2011, pp. 1453--1460.

\bibitem{vit-b16-224}
huggingface, ``Source of vit-b16-224,'' \url{https://huggingface.co/timm/vit_base_patch16_224.augreg2_in21k_ft_in1k}, 2023.

\bibitem{xiangshort}
C.~Xiang, C.~Sitawarin, T.~Wu, and P.~Mittal, ``Short: Certifiably robust perception against adversarial patch attacks: A survey,'' \emph{Proceedings Inaugural International Symposium on Vehicle Security \& Privacy}, 2023.

\bibitem{CostCert}
``Costcert,'' \url{https://github.com/costcert/costcert_project}, 2025.

\end{thebibliography}

\balance

\clearpage

\begin{appendices}








\end{appendices}















\end{document}